\documentclass[11pt]{article}

\usepackage{amsfonts,amssymb,mathrsfs,amsmath,amssymb,amsthm,float}
\usepackage{graphicx}
\usepackage{geometry}
\usepackage{algorithm,algorithmic}
\usepackage{color}
\usepackage{multirow}
\usepackage{tikz}

\newtheorem{theorem}{Theorem}[section]
\newtheorem{prop}[theorem]{Proposition}

\newtheorem{lemma}[theorem]{Lemma}

\newtheorem{remark}[theorem]{Remark}

\newtheorem{example}[theorem]{Example}

\def\N{{\mathbb{N}}}

\def\EP{{\mathbb{E}}}
\def\PR{{\mathbb{P}}}

\def\B{{\mathbb{B}}}
\def\I{{\mathbb{I}}}

\def\R{{\mathbb{R}}}

\def\one{{\mathbf{1}}}

\def\A{{\mathcal{A}}}
\def\C{{\mathcal{C}}}
\def\D{{\mathcal{D}}}

\def\F{{\mathcal{F}}}

\def\G{{\mathcal{G}}}
\def\X{{\mathcal{X}}}
\def\Y{{\mathcal{Y}}}

\def\NN{{\mathcal{N}}}

\def\H{{\mathcal{H}}}
\def\S{{\mathcal{S}}}

\def\Relu{{\hbox{\rm{Relu}}}}
\def\Loss{{\mathbf{L}}}

\def\argmin{\hbox{\rm{argmin}}}
\def\argmax{\hbox{\rm{argmax}}}
\def\mymin{\hbox{\rm{min}}}
\def\mymax{\hbox{\rm{max}}}
\def\maxmin{\hbox{\rm{maxmin}}}
\def\minmax{\hbox{\rm{minmax}}}

\def\myprod{\hbox{$\prod$}}
\def\mysum{\hbox{$\sum$}}

\def\AR{{\rm{AR}}}
\def\AA{{\rm{AA}}}

\def\CE{{\hbox{\scriptsize\rm{ce}}}}
\def\CW{{\hbox{\scriptsize\rm{cw}}}}

\def\MSE{{\hbox{\scriptsize\rm{mse}}}}

\def\Hyp{{\mathcal{H}}}

\def\Be{{\B_{\varepsilon}}}
\def\Bxe{{\B(x,\varepsilon)}}

\begin{document}

%\begin{frontmatter}

\title{
Achieve Optimal Adversarial Accuracy for Adversarial Deep Learning using Stackelberg Game\thanks{This work is partially supported by NSFC grant No.12288201 and NKRDP grant No.2018YFA0704705.}}
\author{
Xiao-Shan Gao, Shuang Liu, and Lijia Yu \\
Academy of Mathematics and Systems Science,  Chinese Academy of Sciences\\
University of  Chinese Academy of Sciences}
%\\Email: xgao@mmrc.iss.ac.cn}
\date{\today}

\maketitle

\begin{abstract}
\noindent
Adversarial deep learning is to train robust DNNs against adversarial attacks,
which is one of the major research focuses of deep learning.
Game theory has been used to answer some of the basic questions
about adversarial deep learning such as the existence of a classifier with optimal robustness
and the existence of optimal adversarial samples  for a given class of classifiers.
In most previous work, adversarial deep learning was formulated as a simultaneous game
and the strategy spaces are assumed to be certain probability distributions in order for the Nash equilibrium to exist. But, this assumption is not applicable to the practical situation.
In this paper, we give answers to these basic questions
for the practical case where the classifiers are DNNs with a given structure,
by formulating the adversarial deep learning as sequential games.
The existence of Stackelberg equilibria for these games are proved.
Furthermore, it is shown that the equilibrium DNN has the largest adversarial accuracy among all DNNs with the same structure, when Carlini-Wagner's margin loss is used.
Trade-off between robustness and accuracy in adversarial deep learning
is also studied from game theoretical aspect.

\end{abstract}

\section{Introduction}

A major safety issue for deep learning~\cite{lecun2015deep} is the existence of adversarial samples~\cite{S2013}, that is, it is possible to  make little modifications to an input sample
which are essentially imperceptible to the human eye,
but the DNN outputs a wrong label or even any label given by the adversary.
Existence of adversarial samples makes
deep learning vulnerable in safety critical applications and
{\em adversarial deep learning}  has becomes a major research focus of deep learning~\cite{sur-adv}.
The goal of adversarial deep learning is to train robust DNNs against adversarial attacks and well as developing more effective attack methods for generating adversarial samples.

Many adversarial defence models were  proposed, including
the adversarial training based on robust optimization~\cite{M2017,trades},
the  gradient masking and obfuscation approaches~\cite{Obfuscated1,yu-biasc},
% which hide the gradient of the DNN to avoid gradient based attacks~\cite{Obfuscated1,yu-biasc},
adversarial parameter attacks~\cite{aq1,pp1,yu-apa},
universal adversaries~\cite{game-aeg1,universal-adv},
randomized smoothing~\cite{smooth1},
and the adversarial sample detection~\cite{asdet1}.
Many attack methods are also proposed, including
the white-box attacks based on gradient information of the DNN~\cite{CW1,M2017,JSMA},
the black-box attacks based on the transferability of the adversaries~\cite{Bbox1},
the poisoning attacks for the input data~\cite{DataPoi1,DataPoi2},
and the physical world attacks~\cite{PhyAt1,PhyAt2}.
More details can be found in the survey~\cite{sur-adv}.

Many of the defenses are found to be susceptible to new adversarial attacks,
and stronger defences also are proposed against the new adversarial attacks.
To break this loop of defences and attacks, a recent line of research based on game theory~\cite{book-game1,book-game2} tries to establish more rigourous foundation for adversarial deep learning  by answering questions such as~\cite{game-aeg1,adl-gm1,game-ne1,game-rd1}:

{\bf Question $\mathbf{Q}_1$:} Does there exists a  classifier which ensures optimal robustness against any adversarial attack?

{\bf Question $\mathbf{Q_2}$:} Does there exist optimal adversarial samples for a given class of classifiers and a given set of data distribution?

To answer these questions, the adversarial deep learning was formulated as
a simultaneous game between the Classifier and the Adversary.
The goal of the Classifier is to train a robust DNN.
The goal of the Adversary is to create optimal adversarial samples.
A {\em Nash equilibrium} of the game is a DNN $\C^*$ and an attack $\A^*$,
such that no player can benefit by unilaterally changing its strategy
and thus gives an optimal solution to the adversarial deep learning.
%
%Deep results were obtained in this line of research.
%
%In \cite{game-aeg1,game-ne1,game-rd1}, the adversarial deep learning was formulated as a
% simultaneous game
%and
Existence of  Nash equilibria was proved under various assumptions \cite{game-aeg1,game-ne1,game-rd1}.
%
%In \cite{adl-gm1}, the adversarial deep learning was formulated as a sequential or Stackelberg game with the Adversary as the leader to generate adversaries.
%
%In above work, the adversarial deep learning is modeled as a non-cooperative game.
%In \cite{coop-game1}, the cooperative game was used
%to explain various adversarial attacks and defenses.

Despite the grerat  progresses, questions $\mathbf{Q_1}$ and $\mathbf{Q_2}$ are not answered satisfactorily.
The main reason is that in order for the Nash equilibrium to exist,
both the Classifier and the Adversary are either assumed to be a convex set of probability distributions
or measurable functions.
However, in practice, DNNs with fixed structures are used and
Nash equilibria do not exist in this case.
In this paper, we will show that questions $\mathbf{Q_1}$ and $\mathbf{Q_2}$
can be answered positively for DNNS with a fixed structure by
formulating the adversarial deep learning as
Stackelberg games.

\subsection{Main contributions}

A positive answer to question $\mathbf{Q_1}$ is given
by formulating the adversarial deep learning
as a Stackelberg game $\G_s$ with the Classifier as the leader
and the Adversary as the follower,
where the strategy space for the Classifier is a class of DNNs with a given structure,
say DNNs with a fixed depth and width.
We show that game $\G_s$ has a Stackelberg equilibrium
which gives the optimal robust DNN under certain robustness measurement
(Refer to Theorem \ref{th-s11}).
We further show that when the Carlini-Wagner margin loss is used as the payoff function,
the equilibrium DNN is the optimal defense which has the largest adversarial accuracy among all DNNs with the same structure (Refer to Theorem \ref{th-ar12}).
Furthermore, the equilibrium DNN is the same as that of the adversarial training~\cite{M2017}.
Thus, our results give another theoretical explanation for
the fact that adversarial training is one of the most effective defences against adversarial attacks.

The trade-off property for deep learning means that
there exists a trade-off  between the robustness and accuracy~\cite{trade1,trade2,trades}.
We prove a trade-off result from game theoretical viewpoint.
Precisely, we show that if  a linear combination
of the payoff functions of adversarial training
and normal training is used as the total payoff function,
then the equilibrium DNN has robustness not higher
and accuracy no lower than that of the DNN obtained by adversarial training.
We also show that trade-off property does not hold
if using empirical loss to train the DNNs,
that is, the DNNs with the largest adversarial accuracy can be parameterized
by elements in an open set of $\R^K$, where $K$ is the number of parameters,
that is, there still exist rooms to improve the accuracy 
for DNNS with the optimal adversarial accuracy.

Finally, when using the empirical loss for a finite set of samples to train the DNN,
we compare $\G_s$ (denoted as $\G_1$ in this case) with two other games:   $\G_2$ is the Stackelberg game
with the Adversary as the leader and  $\G_3$ is the simultaneous game between the Classifier
and the Adversary.
We show that $\G_2$ has a Stackelberg equilibrium  and $\G_3$ has
a mixes strategy Nash equilibrium.
% where the Classifier uses a mixed strategy and the Adversary uses a pure strategy.
Furthermore,  the payoff functions of
$\G_1,\G_2,\G_3$ at their equilibria decrease  successively.
Existence of Stackelberg equilibrium for $\G_2$ gives
a positive answer to question $\mathbf{Q_2}$ for DNNs with a given structure.

\subsection{Related work}

The game theoretical approach to adversarial machine learning
was first studied in the seminal work of Dalvi, Domingos, Mausam, and Verma~\cite{aml-01},
where they formulated adversarial machine learning as a simultaneous game
between the Classifier and the Adversary.
Quite a number of work has been done along this line,
by formulating adversarial machine learning both as
a simultaneous game and as a Stackelberg game,
which can be found in the nice surveys~\cite{aml-sur1,book-game3}.
These works usually used linear models such as SVM
for binary classifications, and used spam email filtering as
the main application background.

Game theoretical approach to adversarial deep learning
appeared recently and was partially stimulated by the fact
that adversarial samples seem inevitable for deep learning~\cite{asulay1,Bast1,Bast2,adv-inev1}.
The adversarial training was introduced in~\cite{M2017}, which is one of the
best practical training method to defend adversaries.
%
%2020
In \cite{game-rd1,game-aeg1,game-latent1,game-ne1,game-pv1,mixedne1,mixedne2}, the adversarial deep learning was all formulated as a simultaneous game.
In \cite{game-rd1}, it was shown that the game exists no pure strategy Nash equilibrium,
but mixed strategies give more robust classifiers.
%
%2020
In \cite{game-aeg1}, it was proved that Nash equilibrium exists
when the strategy space for the Classifier is convex
and the strategy space for the Adversary is certain probability distributions.
%
%2020
In \cite{game-latent1,game-ne1},
it was proved that Nash equilibria exist and can be approximated by a pure strategy,
when the strategy spaces for both the Classifier and Adversary
are parameterized by distributions.
%
%2021
%In \cite{game-ne1}, the adversarial deep learning was formulated as a
%simultaneous game, where the strategies
%for the Classifier and the Adversary are both probability distributions.
%It was proved that Nash equilibria exist.
%
%2020
In \cite{game-pv1}, the Classifier ensures the robustness of a fixed DNN by adding perturbation to the sample to counteract the Adversary.
In \cite{mixedne1,mixedne2}, methods to compute mixed Nash equilibria were given.
In \cite{adl-gm1}, the adversarial deep learning was formulated as a Stackelberg game with the Adversary as the leader, but existence of equilibria was not given.
In \cite{game-Fiez1,game-Jin1}, properties and algorithms for local Stackelberg equilibria were studied.
In above work, the adversarial deep learning is modeled as a non-cooperative game.
In \cite{coop-game1}, the cooperative game is used
to explain various adversarial attacks and defenses.
% from the view of multi-order interactions among  input samples.

Most of the above work formulated adversarial deep learning as a simultaneous
game and assume the strategy spaces to be certain convex probability distributions
in order to prove the existence of the Nash equilibrium.
In this paper, we show that
by formulating the adversarial deep learning as a sequential game,
Stackelberg equilibria exist for DNNs with a given structure,
and the equilibrium DNN is the best defence in that
it has the largest adversarial accuracy among all DNNs with the same structure.
%Furthermore, the Stackelberg equilibrium can be computed with the
%efficient adversarial training~\cite{M2017}.

\vskip 15pt
The rest of this paper is organized as follows.
In section \ref{sec-AT}, preliminary results are given.
In section \ref{sec-s0}, the adversarial deep learning is formulated as
a  Stackelberg game  and the existence of Stackelberg equilibria is proved.
In section \ref{sec-maxaa}, it is proved that adversarial training with
Carlini-Wagner loss gives the best adversarial accuracy.
In section \ref{sec-trade}, two trade-off results are proved.
In section \ref{sec-emp}, three types of adversarial games are
compared when the data set is finite.
In section \ref{sec-conc}, conclusions and problems for further study are given.

\section{Preliminaries}
\label{sec-AT}

\subsection{Adversarial training and robustness of DNN}
Let $\C:\X\to \R^m$ be a classification DNN with $m$ labels in $\Y=[m]=\{1,\ldots,m\}$~\cite{lecun2015deep}.
Without loss of generality, we assume $\X=\I^n$, where $\I=[0,  1]$.
Denote $\C_l(x)\in\R$ to be the $l$-th coordinate of $\C(x)$ for $l\in[m]$, which are called {\em logits} of the DNN.
For $x\in\X$, the classification result of $\C$ is $\widehat{\C}(x) = \argmax_{l\in\Y}\, \C_l(x).$
We assume that Relu is used as the {\em activation function},
so $\C$ is continuous and piecewise linear.
The results are easily generated to any activation functions
which are Lipschitz continuous.

To train a DNN, we need first to choose a {\em hypothesis space} $\Hyp$ for the DNNs,
say the set of CNNs or RNNs with certain fixed structure.
In this paper, denote $\NN_{W,D}$ to be the set of   DNNs
with width $W$ and depth $D$ and   use it as the hypothesis space.
For a given hypothesis space $\Hyp$, the parameter set of DNNs in $\Hyp$ is fixed
and is denoted as $\Theta\in\R^K$, where $K$ is the number of the parameters.
$\C$ can be written as $\C_{\Theta}$ if the parameters need to be mentioned explicitly,
 that is,
\begin{equation}
\label{eq-HS}
\Hyp=\{C_\Theta:\X\to \R^m\,:\,  \Theta\in\R^K\}.
\end{equation}

Let the objects to be classified satisfy a distribution $\D$  over $\X\times\Y$.
Given a loss function $\Loss:\R^m\times\Y\to\R$, the total loss for the data set is
\begin{equation}
\label{eq-LS0}
\varphi_0({\Theta})=
%\rho_0({\Theta},S=((\widetilde{x}_i,y_i))_{i=1}^N\in\S_a)
%\frac1N \sum_{i=1}^N\Loss(\C_{\Theta}({x}_i), y_i).
\EP_{(x,y)\sim\D}\,\Loss(\C_{\Theta}(x), y).
\end{equation}
Training a DNN $\C_{\Theta}$ is to make the total loss minimum by solving the following optimization problem
\begin{equation}
\label{eq-NT00}
\Theta^* = \argmin_{\Theta\in\R^K}\,\varphi_0({\Theta}).
\end{equation}

Given an {\em attack radius} $\varepsilon\in\R_+$,
denote $\Bxe=\{\overline{x}\in\R^n\,:\, ||\overline{x}-x||\le\varepsilon\}$.
We   use $\infty$ norm if not mentioned otherwise.
We will find adversaries for $x$ in $\Bxe$.
Precisely, $\overline{x}\in\Bxe$ is called an {\em adversary of $x$} with label $y$, if
$\widehat{\C}(\overline{x})\ne  y$.
%
%Let  $\Loss:\R^m\times\Y\to\R$ be a loss function.
%
In order to increase the robustness of a trained DNN, the {\em adversarial training}~\cite{M2017}
is introduced which is to solve the following robust optimization problem
\begin{equation}
\label{eq-AT}
\Theta^* = \argmin_{\Theta\in\R^K}\, \EP_{(x,y)\sim \D}\, \mymax_{\overline{x}\in\Bxe}\,
\Loss(\C_{\Theta}(\overline{x}),  y).
\end{equation}
%
%In the inner loop, $\overline{x}$ is computed with the PGD method~\cite{M2017}.
Intuitively, the adversarial training is first computing a {\em most-adversarial sample}
$$x_a = \argmax_{\overline{x}\in\Bxe}\, \Loss(\F(\overline{x}),  l_x)$$
for $x$ and then minimizing $\Loss(\F(x_a),y)$ instead of $\Loss(\F(x),y)$.

Given a DNN $\C$ and an attack radius $\varepsilon$, we define
the {\em adversarial robustness measure} of $\C$ with respect to $\varepsilon$ as follows
\begin{equation}
\label{eq-arm121}
\begin{array}{lcl}
\AR_{\D}(\C,\varepsilon)
&=& \EP_{(x,y)\sim \D}\, \mymax_{\overline{x}\in\B(x,\varepsilon)} \Loss(\C(\overline{x}), y)\\
%\AR_{S}(\C,\varepsilon) &=& \mymax_{A\in\S_a} \varphi_1(\Theta,A).\\
\end{array}
\end{equation}
which is the total loss of $\C$ at the most-adversarial samples.
$\C$ is more robust if $\AR_{\D}(\C,\varepsilon)$ is smaller.
Then the adversarial training is to  find a DNN in $\Hyp$ with the optimal adversarial robustness measurement which is denoted as
\begin{equation}
\label{eq-arm122}
\begin{array}{lcl}
\AR_{\D}(\Hyp,\varepsilon)
&=& \mymin_{\Theta\in\R^K} \AR_{\D}(\C_\Theta,\varepsilon).
%&=& \mymin_{\Theta\in\R^K}
%\EP_{(x,y)\sim \D}\, \mymax_{\widetilde{x}\in\B(x,\varepsilon)}\, \Loss(\C_\Theta(\widetilde{x}), y).\\
%\AR_{S}(\C,\varepsilon) &=& \mymax_{A\in\S_a} \varphi_1(\Theta,A).\\
\end{array}
\end{equation}
$\AR_{\D}(\C,\varepsilon)$ and $\AR_{\D}(\Hyp,\varepsilon)$  have the following simple properties.
%Denote $\NN_{W,D}$ to be the set of certain DNNs with width $W$ and depth $D$.
%Then we have

(1) If $W_1\ge W_2$ and $D_1\ge D_2$, then
$\AR_{\D}(\NN_{W_1,D_1},\varepsilon)\le \AR_{\D}(\NN_{W_2,D_2},\varepsilon)$.

(2) If $\varepsilon_1\le \varepsilon_2$, then
$\AR_{\D}(\C,\varepsilon_1)\le \AR_{\D}(\C,\varepsilon_2)$.

(3) In the optimal case, we have $\AR_{\D}(\C,\varepsilon)=0$, which means that $\C$
gives the correct label for any $\overline{x}\in\B(x,\varepsilon)$.
In this case, we say that $\C$ is
{\em robust} for  the attack radius $\varepsilon$.
It was proved that there exist robust classifiers for a separated  data set~\cite{trade2}.

\subsection{Bounds and continuity of the DNN}
Let $C_\Theta:\X\to \R^m$ be a fully connected feed-forward DNN with depth $D$, whose $l$-th hidden layer is
\begin{equation}
\label{eq-dnn0}
\begin{array}{ll}
 x_{l}=\sigma(W_{l}x_{l-1}+b_{l}) \in \R^{n_{l}}, l=1,  \ldots, D,
\end{array}
\end{equation}
where $n_0=n$, $n_D = m$, $W_{l}\in \R^{n_{l}\times n_{l-1}}$, $b_{l}\in \R^{n_{l}}$,
$\sigma=\Relu$, $x_{0}\in\R^{n}$ is the input, and $x_D\in\R^{m}$ is the output.
The parameter set is $\Theta = \cup_{l=1}^D (W_l\cup b_l)$.
%
%We prove several basic properties for the DNN and the loss function.
It is easy to show that $\C$ is bounded.
For $\varepsilon\in\R_+$, denote  $\I_{\varepsilon}=[-\varepsilon,1+\varepsilon]$.
\begin{lemma}
\label{lm-dnn11}
For any DNN $\C_{\Theta}:\I_{\varepsilon}^n\to\R^m$ with
width $\le W$, depth $\le D$, and   $||\Theta||_2\le E$,
there exists an  $\Omega(n,m,D,W,E,\varepsilon)\in\R_+$ such that $||\C_{\Theta}(x)|| \le \Omega(n,m,D,W,E,\varepsilon)$.
%for $i=1,\ldots,m$, where $\C_{\Theta,i}(x)$ is the $i$-th coordinate
%of $\C_{\Theta}(x)$.
\end{lemma}
\begin{proof}
$\C_{\Theta}(x)$ is bounded because $\C_{\Theta}(x)$ is continuous on $x$ and $\Theta$, and
$[-\varepsilon,1+\varepsilon]^n$ and $[-E,E]^n$ are compact.
$\Omega(n,m,D,W,E,\varepsilon)$ can be derived from \eqref{eq-dnn0}.
\end{proof}

\begin{lemma}
\label{lm-con12}
For any DNN $\C_{\Theta}:\I_{\varepsilon}^n\to\R^m$ with
width $\le W$, depth $\le D$, and $||\Theta||_2\le E$,
there exist   $\Delta(m, n, W,D,E,\varepsilon)$ and $\Lambda(m,n,W,D,E,\varepsilon)\in\R_+$ such that

(1) $||\C_{\Theta}(x)-\C_{\Theta+\alpha}(x)||_2\le\Delta(m,n,W,D,E,\varepsilon)||\alpha||_2$,
that is $\C_{\Theta}(x)$ is Lipschitz on $\Theta$.

(2) $||\C_{\Theta}(x+\delta)-\C_{\Theta}(x)||_2
     \le\Lambda(m,n,W,D,E,\varepsilon)||\delta||$,
that is $\C_{\Theta}(x)$ is Lipschitz on $x$.

\noindent
Thus $\C$ is Lipschitz on $\Theta$ and $x$.
\end{lemma}
\begin{proof}
Without loss of generality, let $\C$ be defined as in \eqref{eq-dnn0}.
Then $\C_{\Theta}(x)=\Theta^D(\cdots\sigma(\Theta^1x)\cdots)$ with $\Theta$ to be the set of all weight matrices, that is, $\Theta=\{\Theta^k|\forall k \in [D]=\{1,2\cdots, D\}\}$ and $\sigma$ is ReLU.
The bias vectors are not considered, which can be included as parts of the
weight matrices by extending the input space slightly, similar to~\cite{N2014}.
We denote $z_k$ and  $\widehat{z}_k$ respectively to be the outputs of the $k$-th hidden layers of $\C_\Theta$ and $\C_{\Theta+\alpha}$, which are $z_k =\sigma(\Theta^k(\cdots\sigma(\Theta^1x)\cdots))$ and $\widehat{z}_k =\sigma(\widehat{\Theta}^k(\cdots\sigma(\widehat{\Theta}^1x)\cdots))$ and $\widehat{\Theta}^i$ is weight matrices of $\C_{\Theta+\alpha}$,  in particular $z_0=\widehat{z}_0\in[-\varepsilon,1+\varepsilon]^n$ is the input.
Since $||\Theta^i-\widehat{\Theta}^i||_2\le ||\alpha||_2$ for any $i\in [D]$
and $|\sigma(a) - \sigma(b)|\le |a-b|$,
we have
\begin{equation*}
\renewcommand{\arraystretch}{1.5}
\begin{array}{ll}
%      &||z_{L}-\widehat{z}_{L}||_2     \\
        &||\C_{\Theta}(x)-\C_{\theta+\alpha}(x)||_2\\
%     & = ||\Theta^D(z_{D-1})- \widehat{\Theta}^D(\widehat{z}_{D-1})||_2 \\
            &= ||(\Theta^D-\widehat{\Theta}^D)z_{D-1}+ \widehat{\Theta}^D (z_{D-1}-\widehat{z}_{D-1})||_2\\
&\le ||\Theta^D-\widehat{\Theta}^D||_2||{z}_{D-1}||_2+ ||\widehat{\Theta}^D||_2||z_{D-1}-\widehat{z}_{D-1}||_2\\
&= ||\Theta^D-\widehat{\Theta}^D||_2||{z}_{D-1}||_2+ ||\widehat{\Theta}^D||_2
||\sigma({\Theta}^{D-1} z_{D-2})-\sigma(\widehat{\Theta}^{D-1} \widehat{z}_{D-2})||_2\\
&
{%\color{red}
\le ||\Theta^D-\widehat{\Theta}^D||_2||{z}_{D-1}||_2+ ||\widehat{\Theta}^D||_2\,
||{\Theta}^{D-1} z_{D-2}-\widehat{\Theta}^{D-1} \widehat{z}_{D-2}||_2
}
\\
%4
            &\le ||\Theta^D-\widehat{\Theta}^D||_2||{z}_{D-1}||_2+ ||\widehat{\Theta}^D||_2(||\Theta^{D-1}-\widehat{\Theta}^
            {D-1}||_2||{z}_{D-2}||_2+ ||\widehat{\Theta}^{D-1}||_2||z_{D-2}-\widehat{z}_{D-2}||_2)\\
&\le ||\Theta^D-\widehat{\Theta}^D||_2||{z}_{D-1}||_2 +
\sum_{k=2}^{D} (\prod_{i=0}^{k-2} ||\widehat{\Theta}^{D-i}||_2)|| \Theta^{D-k+1}-\widehat{\Theta}^{D-k+1}||_2||z_{D-k}||_2\\
&\le(||{z}_{D-1}||_2+ \sum_{k=2}^{D}(\prod_{i=0}^{k-2} ||\widehat{\Theta}^{D-i}||_2)||z_{D-k}||_2) ||\alpha||_2.
        \end{array}
    \end{equation*}
The coefficient $\Delta=(||{z}_{D-1}||_2+ \sum_{k=2}^{D}(\prod_{i=0}^{k-2} ||\widehat{\Theta}^{D-i}||_2)||z_{D-k}||_2)$ is clearly bounded and depends $m, n, W,D,E,\varepsilon$.
Thus $\C_{\Theta}(x)$ is Lipschitz on $\Theta$.
The Lipschitz continuity on $x$ can be proved similarly:
\begin{equation*}
\renewcommand{\arraystretch}{1.5}
        \begin{array}{ll}
            &||\C_{\Theta}(x+\delta)-\C_{\Theta}(x)||_2\\
            &=||\Theta^D(\cdots \sigma\Theta^1(x+\delta)\cdots)- \Theta^D(\cdots \sigma\Theta^1(x)\cdots)||_2\\
            &\le ||\Theta^D||_2||\sigma(\Theta^{D-1}(\cdots\sigma(\Theta^1(x+\delta))\cdots))- \sigma(\Theta^{D-1}(\cdots\sigma(\Theta^1x)\cdots))||_2\\
            &\le ||\Theta^D||_2 ||\Theta^{D-1}(\cdots \sigma\Theta^1(x+\delta)\cdots) - \Theta^{D-1}(\cdots \sigma\Theta^1(x)\cdots)||_2\\
            &\le (\prod_{i=1}^D ||\Theta^i||_2) ||\delta||_2
            \le (\prod_{i=1}^D ||\Theta^i||_2)\sqrt{n} ||\delta||.
        \end{array}
    \end{equation*}
We denote the coefficient as $\Lambda(m,n,W,D,E,\varepsilon)$. The lemma is proved.
We can also extend this result to convolutional neural networks.
\end{proof}

\subsection{Continuity of the loss function}
Unless mentioned otherwise, we assume that the loss function $\Loss(z,y)$ is
continuous on $z\in\R^m$ for a fixed $y\in\Y$.
The mostly often used loss functions have much better properties.
Consider the following loss functions:
the mean square error,
the  crossentropy loss,
and the  margin loss introduced by Carlini-Wagner~\cite{CW1}:
\begin{equation}
\label{eq-loss12}
\begin{array}{ll}
\Loss_\MSE(z,y) = ||z-\one_y||_2^2\\
\Loss_{\CE}(z,y) = \ln(\sum_{i=1}^m \exp(z_i)) - z_{y}\\
\Loss_\CW(z,y)   =\max_{l\in[m],l\ne y}z_{l} -z_{y}\\
\end{array}
\end{equation}
where $\one_y\in\R^m$ is the vector whose $y$-th entry is $1$ and all other entries are $0$.

By Lemma \ref{lm-dnn11}, we can assume that the loss function is defined on a bounded cube:
\begin{equation}
\label{eq-Loss11}
\Loss(z,y):[-B,B]^m\times \Y\to \R
\end{equation}
where $B = \Omega(n,m,D,W,E,\varepsilon)$.
Since $\Y=[m]$ is discrete, we need only consider the
continuity of $\Loss$ on $z$ for a fixed $y$.

\begin{lemma}
\label{lm-con11}
For a fixed $y$, all three loss functions in \eqref{eq-loss12} are
Lipschitz continuous on $z$ over $[-B,B]^m$,
with Lipschitz constants $2\sqrt{m}\max\{B,1\}$, $\sqrt{2}$, $\sqrt{2}$, respectively.
\end{lemma}
\begin{proof}
It suffices to show that $||\nabla_z F(z)||_2\le V$ is bounded over $[-B,B]^m$.
For a fixed $y$, let $f(z)=\Loss(z,y)$.
Then from  $||\nabla_z F(z)||_2\le V$,  by the mean value theorem and the Schwarz inequality,
%if a function $f:[a,b]-\tp\R$ is continuous over $[a,b]$ and
%differentiable on $(a,b)$, then there exists a $c$ such that $f(b) - f(a) = f'(c)(b-a)$.
%
we have $||F(z+\delta) -F(z)||_2  =
||F'(z_1) \delta||_2
\le ||F'(z_1)||_2 ||\delta||_2  \le V||\delta||_2$,
where $z_1\in(-B,B)^m$. Thus $\Loss$ is Lipschitz with constant $V$.

For $\Loss_\MSE$, we have $||\nabla_z\Loss_\MSE(z, y)||_2= 2||(z-\one_y)||_2\le 2\sqrt{m}\max\{B,1\}$.
For $\Loss_{\CE}$, we have
$||\nabla_z\Loss_\CE(z, y)||_2
=\sqrt{\frac{\sum_{i=1\,i\ne y}^m \exp(2z_i)+(\sum_{i=1\,i\ne y}^m \exp(z_i))^2}{(\sum_{i=1}^m \exp(z_i))^2}} \le \sqrt{2}$.
For $\Loss_{\CW}$, we have
$||\nabla_z\Loss_\CW(z, y)||_2 =\sqrt{2}$.
The lemma is proved.
\end{proof}
%Note that Lipschitz continuous is stronger than uniformly continuous.

\section{Adversarial training as a Stackelberg game}
\label{sec-s0}
In this section, we formulate the adversarial deep learning
as a Stackelberg game and
prove the existence of the Stackelberg equilibria.
%and show that the output is the same as that of adversarial training \cite{M2017}.

\subsection{Stackelberg game}
\label{sec-s1}

Consider a two-player zero-sum minmax sequential or  Stackelberg game $\G=(\S_L,\S_F,\varphi)$,
where $\S_L$ and $\S_F$  are respectively the strategy  spaces for the leader and the follower of the game and $\varphi:\S_L\times\S_F\to\R$ is the payoff function.

In the Stackelberg game $\G$, the leader moves first by picking a strategy
$s_l\in\S_L$ to minimize the payoff, knowing   the existence of the follower.
After knowing   $s_l$, the follower
picks $s_f\in\S_F$ to maximize the payoff.
Formally, $(s_l^*,s_f^*)\in\S_L\times\S_F$
is called  a {\em Stackelberg equilibrium}  of $\G$  if
\begin{equation}
\label{eq-gamma}
\begin{array}{lcl}
\gamma(s_l)&=&\{  \argmax_{s_f\in\S_F} \varphi(s_l, s_f)\}\subset\S_F\\
\end{array}
\end{equation}
is not empty for any $s_l\in\S_L$, and
\begin{equation}
\label{eq-SE}
\begin{array}{lcl}
s_l^*
%&\in&
%\argmin_{\Theta\in\S_c}\,
%\varphi(\Theta,\gamma(\Theta))\\
&\in&
\argmin_{s_l\in\S_L,S(s_l)\in\gamma(s_l)}\, \varphi(s_l,S(s_l)) \hbox{ and }
s_f^*
\in
\argmax_{s_f\in\S_F}\, \varphi(s_l^*, s_f)= \gamma(s_l^*).\\
\end{array}
\end{equation}
Let
\begin{equation}
\label{eq-Gamma}
\begin{array}{ccl}
%(s_l^*,s_f^*)
%&\in&
%\argmin_{(s_l,s_f)\in\Gamma}\, \varphi(s_l,s_f),\hbox{ where}\\
%
\Gamma &=& \{(s_l,s_f)\,:\, s_l\in\S_L, s_f\in\gamma(s_l)\}.\\
\end{array}
\end{equation}
Then, \eqref{eq-SE} is equivalent to
%\begin{equation}
%\label{eq-Gamma10}
%\begin{array}{ccl}
$(s_l^*,s_f^*)
\in
\argmin_{(s_l,s_f)\in\Gamma}\, \varphi(s_l,s_f).$
%
%\Gamma &=& \{(s_l,s_f)\,:\, s_l\in\S_L, s_f\in\gamma(s_l)\}.\\
%\end{array}
%\end{equation}
%
We have the following  result.
\begin{theorem}[\cite{StackelbergNE1}]
\label{th-SE}
If the strategy spaces are compact and the payoff function is continuous,
then  the sequential game $\G$ has a Stackelberg equilibrium,
which is also a
{\em subgame perfect Nash equilibrium} of game $\G$
as an extensive form game~\cite{book-game1}.
\end{theorem}

\subsection{Adversarial training as a Stackelberg game}
\label{sec-s2}

We formulate adversarial deep learning  as a two-player zero-sum minmax Stackelberg game $\G_s$,
which is the best defence for adversarial deep learning in certain sense.

{\bf The leader of the game is the Classifier}, whose goal is to train a robust DNN
$\C_\Theta:\I^n\rightarrow\R^m$ in the hypothesis space $\Hyp$ in \eqref{eq-HS}.
%
%$\C$ could be any DNN such as a fully connected DNN, CNN, or Resnet,
%but we assume that the structure of $\C$ is fixed and the number of
%parameters of $\C$ is $K$.
%
%By Lemma \ref{lm-dnn11},  we assume that the parameters of $\C$ are in
%\begin{equation}
%\label{eq-Sc}
%\S_c= \U^K \hbox{ where } \U=[-1,1],
%\end{equation}
Without loss of generality, we assume that the parameters of $\C$ are in
\begin{equation}
\label{eq-Sc}
\S_c= [-E,E]^K
\end{equation}
for some $E\in\R_+$,
that is, the {\em strategy space} for the Classifier is $\S_c$.
%Denote $\NN_c=\{\C_\Theta\,:\, \Theta\in\S_c\}$ to be
%the corresponding DNNs.

{\bf The follower of the game is the Adversary},
whose goal is to create the best adversary within a given attack radius $\varepsilon\in\R_+$.
The strategy space for the Adversary is
\begin{equation}
\label{eq-Sa}
\S_a=\{ A : \X\to \Be\}
\end{equation}
where $\Be =\{\delta\in\R^n\,:\, ||\delta||\le \varepsilon\}$
is the ball  with the origin point as the center and $\varepsilon$ as the radius.
By considering the $L_\infty$ norm, $\S_a$ becomes a metric space.

{\bf The payoff function}.
Given $\Theta\in\S_c$ and  $A\in\S_a$,
the payoff function is the expected loss
\begin{equation}
\label{eq-LS1}
\varphi_s(\Theta,A)=
\EP_{(x,y)\sim \D}\, \Loss(\C_{\Theta}(x+A(x)), y).
\end{equation}
%The goal of the Classifier is minimize  $\varphi$
%and the goal of the Adversary is to maximize $\varphi$.
From \eqref{eq-Loss11}, the composition of $\Loss$ and
$\C_{\Theta}(x+A(x))$ is well-defined,
since $||A(x)||\le \varepsilon$.

For game $\G_s$, $\gamma$ and $\Gamma$ defined in \eqref{eq-gamma} and \eqref{eq-Gamma} are
\begin{equation}
\label{eq-gamma10}
\begin{array}{ccl}
\gamma_s(\Theta)&=&\{  \argmax_{A\in\S_a} \varphi_s(\Theta, A)\}\hbox{ for }\Theta\in\S_c\\
\Gamma_s &=& \{(\Theta,A)\,:\, \Theta\in\S_c, A\in\gamma_s(\Theta)\}\\
\end{array}
\end{equation}
and $(\Theta_s^*,A_s^*)$ is a Stackelberg equilibrium of $\G_s$ if
\begin{equation}
\label{eq-SEs}
\begin{array}{lcl}
\Theta_s^* \in
\argmin_{\Theta\in\S_c,A(\Theta)\in\gamma_s(\Theta)}\, \varphi_s(\Theta,A(\Theta)) \hbox{ and }
A_s^*
\in
\argmax_{A\in\S_a}\, \varphi_s(\Theta_s^*, A).\\
\end{array}
\end{equation}
%
%Denote $\gamma_s,\Gamma_s$ to be the $\gamma,\Gamma$
%defined in \eqref{eq-gamma} and \eqref{eq-Gamma}
%for game $\G_s$, respectively.

\begin{lemma}
\label{lm-s11}
$\varphi_s(\Theta,A):\S_c \times \S_a \to \R$ defined in \eqref{eq-LS1} is a continuous and bounded function.
\end{lemma}
\begin{proof}
It is clear that $\varphi_s(\Theta,A)$ is continuous on $\Theta$,
since $\Loss$ is continuous on $z$ and $\C_\Theta$ is continuous on $\Theta$.
Denote $\phi(x) = \Loss(\C_{\Theta}(x), y):\I_{\varepsilon}^n\to\R$ for fixed $\Theta$ and $y$.
%Then $\phi(x)$ is continuous.
%Since $\I^n$ is closed and bounded,
Then $\phi(x)$ is uniformly continuous by Lemmas \ref{lm-con12} and \ref{lm-con11}.
Given an $A_0\in \S_a$ and $\epsilon>0$,
since $\phi(x)$ is uniformly continuous, there exists a $\delta>0$
such that for $A(x)\in\S_a$ satisfying $||A_0(x)-A(x)||_{\infty} < \delta$,
we have $|\phi(x+A_0(x)) -\phi(x+A(x))|<\epsilon$ for all $x\in\X$.
%
%
% for any $A\in \S_a$, then $\forall \eta >0$, there exist $\delta >0 $, if $|A(x) - A_0(x)|\le \delta$ for some $x$, then $|\Loss(\C_{\Theta}(x+A(x)), y)- \Loss(\C_{\Theta}(x+A_0(x)), y)|\le \epsilon$.
%Let $||A - A_0||\le \delta)$(i.e. $\max\limits_{x\in \X}|A(x) - A_0(x)|\le \delta$).
Then
    \begin{align*}
        |\varphi_s(\Theta, A)-\varphi_s(\Theta, A_0)| &= |\EP_{(x,y)\sim \D}\, [\Loss(\C_{\Theta}(x+A(x)), y)- \Loss(\C_{\Theta}(x+A_0(x)), y)]| \\
        &\le \EP_{(x,y)\sim \D} |\Loss(\C_{\Theta}(x+A(x)), y)- \Loss(\C_{\Theta}(x+A_0(x)), y)|\\
        &\le \epsilon.
    \end{align*}
Hence $\varphi_s(\Theta, A)$ is continuous on $\S_a$.
By Lemma \ref{lm-dnn11},  $\varphi_s(\Theta, A)$ is bounded,
since $||A(x)||\le \varepsilon$.
\end{proof}

\begin{lemma}
\label{lm-s12}
%Let $\gamma(\Theta)$ be defined in \eqref{eq-gamma11}. Then
 $\gamma_s(\Theta)\ne\emptyset$ and $A^*\in\gamma_s(\Theta)$ if and only if
$A^*(x) \in\{\argmax_{A(x)\in\Be}\,\Loss(\C_{\Theta}(x+A(x)), y) \}$
for all $(x,y)\sim\D$.
\end{lemma}
\begin{proof}
We have
    \begin{equation*}
        \begin{array}{lcl}
             \max_{A\in \S_a}\varphi_s(\Theta, A)
             &=& \max_{A\in \S_a}\EP_{(x,y)\sim \D}\, \Loss(\C_{\Theta}(x+A(x)), y)\\
             &\le& \EP_{(x,y)\sim \D}\, \max\limits_{A(x)\in \B_\epsilon}\Loss(\C_{\Theta}(x+A(x)), y).
        \end{array}
    \end{equation*}
Since $\Loss(C(x), y)$ is continuous on $x$ and $\B_\epsilon$ is compact, for every $(x, y)$, $\argmax_{A(x)\in \B_\epsilon}\Loss(\C_{\Theta}(x+A(x)), y)$ exists.
Thus, by choosing these maximum values, we obtain an $A^*\in \S_a$,
which achieves  $\max_{A\in \S_a}\varphi_s(\Theta, A)$.
The lemma is proved.
\end{proof}

\begin{lemma}
\label{lm-s13}
%Let $\Gamma = \{(\Theta,A)\,:\, \Theta\in\S_c, A\in\gamma(\Theta)\}$.
$\Gamma_s$  is a closed set in $\S_c\times\S_a$.
\end{lemma}
\begin{proof}
Let $(\Theta_i,A_i)_{i=1}^{\infty}\in\Gamma_s$ converse to $(\Theta_0,A_0)$.
Supposing $(\Theta_0,A_0)\not\in\Gamma_s$, we will obtain a contradiction.
By Lemma \ref{lm-s12}, there exists a $(\Theta_0,A^*)\in\Gamma_s$,
and thus, $\varphi_s(\Theta_0,A^*)> \varphi_s(\Theta_0,A_0)$ by \eqref{eq-gamma10}.
Let $\eta = \varphi_s(\Theta_0,A^*)- \varphi_s(\Theta_0,A_0)>0$.
By Lemma \ref{lm-s11},  $\varphi_s$ is continuous.
Then there exists an $i_0$ such that
$|\varphi_s(\Theta_{i_0},A_{i_0})- \varphi_s(\Theta_0,A_0)| < \eta/3$
and
$|\varphi_s(\Theta_{i_0},A^*)- \varphi_s(\Theta_0,A^*)| < \eta/3$.
We thus have
$$\varphi_s(\Theta_{i_0},A_{i_0})
< \varphi_s(\Theta_0,A_0) + \eta/3
= \varphi_s(\Theta_0,A^*) -\frac{2\eta}{3}
< \varphi_s(\Theta_{i_0},A^*) -\eta/3 < \varphi_s(\Theta_{i_0},A^*)$$
which contradicts to  $(\Theta_{i_0},A_{i_0}) \in\Gamma_s$
meaning that $\varphi_s(\Theta_{i_0},A_{i_0})\ge \varphi_s(\Theta_{i_0},A)$
for any $A\in\S_a$.
The lemma is proved.
\end{proof}

We have
\begin{theorem}
\label{th-s11}
Game $G_s$ has a Stackelberg equilibrium $(\Theta_s^*, A_s^*)$.
%   and
%$\varphi_s(\Theta_s^*, A_s^*)$ are the same for all Stackelberg equilibria.
Furthermore,
$\Theta_s^*$ is the solution to the adversarial training in \eqref{eq-AT}.
\end{theorem}
\begin{proof}
By Lemma \ref{lm-s11},  $\varphi_s(\Theta,A)$ is bounded.
Then
$\alpha%=\inf_{\Theta\in\S_c,A\in\gamma(\Theta)}\,\varphi_s(\Theta, A)
=\inf_{(\Theta,A)\in\Gamma_s}\,\varphi_s(\Theta, A)$ exists and is finite.
%
%Let $\Gamma$ be defined in \eqref{eq-gamma11}.
%
There exist $(\Theta_i,A_i)_{i=1}^{\infty}\in\Gamma_s$ such that
$\varphi_s(\Theta_i, A_i)$ converges to $\alpha$.
Since $\S_c$ is compact, we can assume that
$\Theta_i$ converges to $\Theta_0$.
Then there exists an $A_0\in\S_a$ such that $(\Theta_0, A_0)\in\Gamma_s$.

We claim that
$\varphi_s(\Theta_i, A_i)$ converges to $\varphi_s(\Theta_0, A_0)$.
Suppose the contrary, that is, $\varphi_s(\Theta_0, A_0) > \alpha$.
Then there exists an $\eta > 0$ such that $\varphi_s(\Theta_0, A_0) > \alpha + \eta$.
Since $\varphi_s(\Theta_i, A_i)$ converges to $\alpha$,
$\exists K_1\in\N_{+}$ such that $\varphi_s(\Theta_k, A_k) < \alpha + \frac{\eta}{3}$ for $\forall k > K_1$.
Since $\varphi_s(\Theta, A)$ is continuous on $\Theta$,
$\exists K_2\in\N_{+}$ such that $\varphi_s(\Theta_k, A_0)> \varphi_s(\Theta_0,A_0)-\frac{\eta}{3}$ for $\forall k > K_2$.
Then for $k>\max\{K_1,K_2\}$, we have
\begin{equation*}
    \begin{array}{lcl}
    \varphi_s(\Theta_k, A_0)
    > \varphi_s(\Theta_0, A_0)-\frac{\eta}{3}
    >\alpha +\frac{2\eta}{3} > \varphi_s(\Theta_k, A_k) + \frac{\eta}{3}
    > \varphi_s(\Theta_k, A_k)
    \end{array}
\end{equation*}
%The first inequality is from the continuity of $\varphi$ on $\Theta$.
which contradicts to   $(\Theta_k, A_k)\in \Gamma_s$. Then $(\Theta_0, A_0)$ is a Stackelberg equilibrium of game $\G_s$.
%
%By Lemma \ref{lm-s13}, $\Gamma$ is closed.
%Then there exist $(\Theta,A)\in \Gamma$ such that
%$\alpha=\varphi_s(\Theta, A)$ and Stackelberg equilibrium exist for game $\G_s$.

Let $(\Theta_s^*, A_s^*)$ be a Stackelberg equilibria of game $\G_s$.
By  Lemma \ref{lm-s12},
\begin{equation*}
\label{eq-se13}
\begin{array}{lcl}
\Theta_s^*
&\in&
\argmin_{\Theta\in\S_c, A(\Theta)\in\gamma(\Theta)}\,
\varphi_s(\Theta,A(\Theta))\\
&=&
\argmin_{\Theta\in\S_c, A_{\Theta}\in\gamma(\Theta)}\,
\EP_{(x,y)\sim \D}\, \Loss(\C_{\Theta}(x+A_{\Theta}(x)), y)\\
&\in&
\argmin_{\Theta\in\S_c}\,
\EP_{(x,y)\sim \D}\, \max_{A_{\Theta}(x)} \Loss(\C_{\Theta}(x+A_{\Theta}(x)), y)\\
&=&
\argmin_{\Theta\in\S_c}\,
\EP_{(x,y)\sim \D}\, \max_{\overline{x}\in\Bxe} \Loss(\C_{\Theta}(\overline{x}), y).\\
\end{array}
\end{equation*}
Briefly,
\begin{equation}
\label{eq-pr-31}
\begin{array}{lcl}
\Theta_s^*
%&=&
%\argmin_{\Theta\in\S_c,A_s^*(\Theta)\in\A_*(\Theta)}\,  \varphi_s(\Theta,A_s^*(\Theta))\\
% \sum_{i=1}^N \Loss(\C_{\Theta}(\rho_i( {x}_i^*(\Theta)), y_i).
%
&=&
\argmin_{\Theta\in\S_c}\, \varphi_s(\Theta, \argmax_{A\in S_a}\, \varphi_s(\Theta,A))\\
&=&
\argmin_{\Theta\in\S_c}\, \mymax_{A\in S_a}\, \varphi_s(\Theta,A).\\
%
%A_s^* &\in& \argmax_{A\in \S_a}\, \varphi_s(\Theta_s^*,A)\nonumber\\
\end{array}
\end{equation}
That is, $\Theta_s^*$ is the solution to the adversarial training \eqref{eq-AT}.
\end{proof}

\begin{remark}
\label{rem-s10}
As a consequence of Theorem \ref{th-s11},
the Stackelberg game $\G_s$  gives the {\em best defence}
in the hypothesis space $\Hyp$ for a given attack radius,
if using $\AR_\D$ in  \eqref{eq-arm122} to measure the robustness.
Precisely, let $(\Theta_s^*, A_s^*)$ be  a Stackelberg equilibrium of game $G_s$.
Then $\AR_\D(\C_{\Theta_s^*},\varepsilon) = \AR_\D(\Hyp,\varepsilon)$.
\end{remark}

\subsection{Refined properties of $\Gamma_s$}
\label{sec-gamma}

In the general case, $\gamma_s(\Theta)$ defined in \eqref{eq-gamma10} may have more than one elements.
In this section, we will prove that
if $\gamma_s(\Theta)$ contains a unique element,
then $\Gamma_s$  defined in \eqref{eq-gamma10} is compact,
which will be used in section \ref{sec-emp}.
%
%Note that $\Gamma_s$ is already proved to be closed in Lemma \ref{lm-s13}.

{\noindent\bf Assumption $A_1$}.
For any $\Theta\in\S_c$, $\gamma_s(\Theta)=\{A^*(\Theta)\}$ defined in \eqref{eq-gamma10} has a unique element and the loss function $\Loss$ is Lipschitz.

\begin{remark}
\label{rem-uni}
Assumption ${A}_1$ is   true in the generic case.
By Lemma \ref{lm-s12}, $A^*\in\gamma_s(\Theta)$ if and only if
$A^*(x) \in\{\argmax_{A\in\Be}\,\Loss(\C_{\Theta}(x+A), y) \}$.
Then Assumption ${A}_1$ is   true if and only if
$\argmax_{A\in\Be}\,\Loss(\C_{\Theta}(x+A), y)$ has a unique solution.
Suppose the loss function is $\Loss_\CW$.
Then $\phi(A)=\Loss(\C_{\Theta}(x+A), y)$ is a piecewise linear function in $A$
and its graph over $\Be$ is a polyhedron as illustrated in Figure \ref{fig-poly1}.
Then its maximum can be achieved only at the vertex of the polyhedron
or the intersection of the $(n-1)$-dimensional sphere $||x-x_0||=\varepsilon$ and the one dimensional edges of the polyhedron.
In the generic case, that is, when the parameters are sufficiently general (refer to Assumption 3.1 in \cite{yu-l2} for more details), there exists only one maximum.
\end{remark}

%%%%%%%%%%%%%%%%%%%%%%%%%%
\tikzset{every picture/.style={line width=0.75pt}} %set default line width to 0.75pt
\begin{figure}[H]
\centering
\begin{tikzpicture}
[x=0.75pt,y=0.75pt,yscale=-0.45,xscale=0.45]

%Shape: Square [id:dp710106432973264]
\draw   (181,17) -- (450,17) -- (450,286) -- (181,286) -- cycle ;
%Shape: Triangle [id:dp4121843788042845]
\draw   (321.53,96) -- (358,177) -- (247,177) -- cycle ;
%Straight Lines [id:da6321258560469549]
\draw    (181,17) -- (321.53,96) ;
%Straight Lines [id:da5662185373163213]
\draw [fill={rgb, 255:red, 0; green, 0; blue, 0 }  ,fill opacity=1 ]   (321.53,96) -- (450,17) ;
%Straight Lines [id:da6279571926307821]
\draw    (358,177) -- (450,286) ;
%Straight Lines [id:da8771382703004194]
\draw    (247,177) -- (181,286) ;
%Straight Lines [id:da20256855169582422]
\draw    (247,177) -- (181,17) ;
%Shape: Circle [id:dp5280031697879783]
\draw   (215,145.5) .. controls (215,94.97) and (255.97,54) .. (306.5,54) .. controls (357.03,54) and (398,94.97) .. (398,145.5) .. controls (398,196.03) and (357.03,237) .. (306.5,237) .. controls (255.97,237) and (215,196.03) .. (215,145.5) -- cycle ;
%Shape: Circle [id:dp23152540880970296]
\draw  [fill={rgb, 255:red, 0; green, 0; blue, 0 }  ,fill opacity=1 ] (309,148) .. controls (309,146.62) and (307.88,145.5) .. (306.5,145.5) .. controls (305.12,145.5) and (304,146.62) .. (304,148) .. controls (304,149.38) and (305.12,150.5) .. (306.5,150.5) .. controls (307.88,150.5) and (309,149.38) .. (309,148) -- cycle ;
%Shape: Circle [id:dp12200493009886193]
\draw  [fill={rgb, 255:red, 0; green, 0; blue, 0 }  ,fill opacity=1 ] (324.03,96) .. controls (324.03,94.62) and (322.91,93.5) .. (321.53,93.5) .. controls (320.15,93.5) and (319.03,94.62) .. (319.03,96) .. controls (319.03,97.38) and (320.15,98.5) .. (321.53,98.5) .. controls (322.91,98.5) and (324.03,97.38) .. (324.03,96) -- cycle ;
%Shape: Circle [id:dp30616904031561165]
\draw  [fill={rgb, 255:red, 0; green, 0; blue, 0 }  ,fill opacity=1 ] (360.5,177) .. controls (360.5,175.62) and (359.38,174.5) .. (358,174.5) .. controls (356.62,174.5) and (355.5,175.62) .. (355.5,177) .. controls (355.5,178.38) and (356.62,179.5) .. (358,179.5) .. controls (359.38,179.5) and (360.5,178.38) .. (360.5,177) -- cycle ;
%Shape: Circle [id:dp41549500640337755]
\draw  [fill={rgb, 255:red, 0; green, 0; blue, 0 }  ,fill opacity=1 ] (249.5,177) .. controls (249.5,175.62) and (248.38,174.5) .. (247,174.5) .. controls (245.62,174.5) and (244.5,175.62) .. (244.5,177) .. controls (244.5,178.38) and (245.62,179.5) .. (247,179.5) .. controls (248.38,179.5) and (249.5,178.38) .. (249.5,177) -- cycle ;
%Shape: Circle [id:dp8651843088648397]
\draw  [fill={rgb, 255:red, 0; green, 0; blue, 0 }  ,fill opacity=1 ] (364,72.33) .. controls (364,70.95) and (362.88,69.83) .. (361.5,69.83) .. controls (360.12,69.83) and (359,70.95) .. (359,72.33) .. controls (359,73.71) and (360.12,74.83) .. (361.5,74.83) .. controls (362.88,74.83) and (364,73.71) .. (364,72.33) -- cycle ;
%Shape: Circle [id:dp3187690335375548]
\draw  [fill={rgb, 255:red, 0; green, 0; blue, 0 }  ,fill opacity=1 ] (267.36,64) .. controls (267.36,62.62) and (266.24,61.5) .. (264.86,61.5) .. controls (263.48,61.5) and (262.36,62.62) .. (262.36,64) .. controls (262.36,65.38) and (263.48,66.5) .. (264.86,66.5) .. controls (266.24,66.5) and (267.36,65.38) .. (267.36,64) -- cycle ;
%Shape: Circle [id:dp10731380845700067]
\draw  [fill={rgb, 255:red, 0; green, 0; blue, 0 }  ,fill opacity=1 ] (223.7,113.67) .. controls (223.7,112.29) and (222.58,111.17) .. (221.2,111.17) .. controls (219.81,111.17) and (218.7,112.29) .. (218.7,113.67) .. controls (218.7,115.05) and (219.81,116.17) .. (221.2,116.17) .. controls (222.58,116.17) and (223.7,115.05) .. (223.7,113.67) -- cycle ;
%Shape: Circle [id:dp28840960539302674]
\draw  [fill={rgb, 255:red, 0; green, 0; blue, 0 }  ,fill opacity=1 ] (381.7,201.67) .. controls (381.7,200.29) and (380.58,199.17) .. (379.2,199.17) .. controls (377.81,199.17) and (376.7,200.29) .. (376.7,201.67) .. controls (376.7,203.05) and (377.81,204.17) .. (379.2,204.17) .. controls (380.58,204.17) and (381.7,203.05) .. (381.7,201.67) -- cycle ;
%Shape: Circle [id:dp04170960626288589]
\draw  [fill={rgb, 255:red, 0; green, 0; blue, 0 }  ,fill opacity=1 ] (235.7,199.67) .. controls (235.7,198.29) and (234.58,197.17) .. (233.2,197.17) .. controls (231.81,197.17) and (230.7,198.29) .. (230.7,199.67) .. controls (230.7,201.05) and (231.81,202.17) .. (233.2,202.17) .. controls (234.58,202.17) and (235.7,201.05) .. (235.7,199.67) -- cycle ;
%
% Text Node
\draw (309,148) node [anchor=north west][inner sep=0.75pt]   [align=left] {$\displaystyle x_{0}$};
\end{tikzpicture}
\caption{Illustration for the graph of $\Loss(\C(x+A),y)$ as a function of $x$ and $A$.}
\label{fig-poly1}
\end{figure}
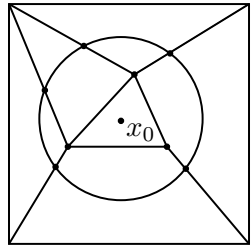

We first introduce three notations which will be used in this section.
By Lemma \ref{lm-con11}, $\Loss(z, y)$ is Lipschitz for $z$ over $[-B,B]^m$
when the loss functions in \eqref{eq-loss12} are used,
and let $\Psi$ be the Lipschitz constant.
By Lemma \ref{lm-con12}, $\C_{\Theta}(x)$ is Lipschitz for $\Theta$ and $x$,
and let $\Delta$ and $\Lambda$ be the Lipschitz constants, respectively.

\begin{lemma}
\label{lm-con13}
%$\Loss(\C_{\Theta}(x+A(x)), y)$ and
For any $\C_{\Theta}:\I_{\varepsilon}^n\to\R^m$  and $\D$, $\varphi_s(\Theta,A)$  defined in \eqref{eq-LS1}
is Lipschitz on $\Theta$ and $A$ when the loss function is Lipschitz.
\end{lemma}
\begin{proof}
Firstly, consider  $\varphi_s(\Theta, A)$ for any fixed $A$.
%
%By Lemma \ref{lm-dnn12}, $\C_{\Theta}$ is bounded.
%, that is, $|\C_{\Theta,i}|\le D$ for $i\in[m]$,
%where $\C_{\Theta,i}$ is the $i$-th coordinate of $\C_{\Theta}$.
%
For any $\epsilon >0$, let $\delta=\frac{\epsilon}{\Psi\Delta}$. Then for any $\Theta_1, \Theta_2$ satisfying $||\Theta_1-\Theta_2||_2\le \delta$, we have
\begin{equation*}
\renewcommand{\arraystretch}{1.4}
\begin{array}{ll}
|\varphi_s(\Theta_1, A)-\varphi_s(\Theta_2, A)|
&=|\EP_{(x, y)\sim \D}[\Loss(\C_{\Theta_1}(x+A(x)), y) - \Loss(\C_{\Theta_2}(x+A(x)), y)]|\\
&\le \EP_{(x, y)\sim\D}\,\Psi||\C_{\Theta_1}(x+A(x))-\C_{\Theta_2}(x+A(x))||_2\\
&\le \EP_{(x, y)\sim\D}\,\Psi\Delta||\Theta_1-\Theta_2||_2
%
%            &\le \EP_{(x, y)\max_{(x, y)\sim\D}}|\Loss(\C_{\Theta_1}(x+A(x)), y) -\Loss(\C_{\Theta_2}(x+A(x)), y)|\\
%            &\le \max_{(x, y)\sim\D}\Psi ||\C_{\Theta_1}(x+A(x))-\C_{\Theta_2}(x+A(x))||_2\\
%            &\le \Psi\Delta ||\Theta_1-\Theta_2||_2
            \le \epsilon
        \end{array}
    \end{equation*}
that is, $\varphi_s(\Theta, A)$ is Lipshitz continuous on $\Theta$.
The proof for the Lipschitz continuity on $A$ is  similar.
\end{proof}

%In the following content, we always denote the Lipschitz constant of $\Loss$ is $\Psi$.
\begin{lemma}
\label{lm-con14}
For $\Theta_i\in\S_c$, if $\lim_{i\to\infty}\Theta_i=\Theta_0$ and $g_i\in \gamma_s(\Theta_i)$,
then for any $(x,y)\sim\D$, the limit of any convergent subsequence of  $\{g_i(x)\}_{i=1}^\infty$ belongs to $\argmax_{A\in\B_{\epsilon}}L(\C_{\Theta_0}(x+A),y)$.
\end{lemma}
\begin{proof}
The result can be proved similar to that of Lemma \ref{lm-s13}.
\end{proof}
%    Assume that subsequence of  $\{g_i(x)\}_{i=1}^\infty$ conveges to $\delta_0$. Suppose $\delta_0 \notin \argmax_{\delta\in\B_{\epsilon}}L(\C_{\Theta_0}(x+\delta),y)$. Since $\Loss$ is continuous on $x$ and $\B_\epsilon$ is compact, there exists a $\delta^*\in \argmax_{\delta\in\B_{\epsilon}}L(\C_{\Theta_0}(x+\delta),y)$. Let $\Loss(\C_{\Theta_0}(x+\delta^*), y)- \Loss(\C_{\Theta_0}(x+\delta_0), y)=\eta$, by assumption $\eta>0$.
%    Because $\Loss$ is continuous on $\Theta, x$, then there exists a sufficient large $k\in \N_{+}$ such that $|\Loss(\C_{\Theta_k}(x+g_k(x)), y)]- \Loss(\C_{\Theta_0}(x+\delta_0), y)]|\le\frac{\eta}{3}$ and $|\Loss(\C_{\Theta_k}(x+\delta^*), y)]- \Loss(\C_{\Theta_0}(x+\delta^*), y)]|\le\frac{\eta}{3}$.Thus:
%    \begin{equation*}
%    \begin{array}{ll}
%        \Loss(\C_{\Theta_k}(x+g_k(x), y) &\le \Loss(\C_{\Theta_0}(x+\delta_0),y) + \frac{\eta}{3} = \Loss(\C_{\Theta_0}(x+\delta^*), y)] -\frac{2\eta}{3} \\
%        &\le\Loss(\C_{\Theta_k}(x+\delta^*), y) -\frac{\eta}{3} \\
%        & < \Loss(\C_{\Theta_k}(x+\delta^*), y)
%    \end{array}
%    \end{equation*}
%    which contradict to the definition of $g_k(x)$ meaning that  $g_k(x)\in \argmax_{\delta\in\B_{\epsilon}}L(\C_{\Theta_k}(x+\delta),y)$. Hence $\delta_0 \in \argmax_{\delta\in\B_{\epsilon}}L(\C_{\Theta_0}(x+\delta),y)$.
%\end{proof}

\begin{lemma}
\label{lm-con21}
Under Assumption ${A}_1$,
%. Assume $\Loss$ is Lipschitz continuous with Lipschitz constant $C$ and
%$\gamma(\Theta)=\{A^*(\Theta)\}$ has a unique element.
%
for any $\Theta \in \S_c$, $A^*(\Theta)(x)$ is continuous on $x$.
\end{lemma}
\begin{proof}
Let $\{(x_i,y_i)\}_{i=1}^{\infty}\subset\X\times\Y$ converges to $(x_0,y_0)$.
Since $\Y$ is finite, we may assume $y_i=y_0$ for all $i$.
Then for any $\Theta$, we will prove $\lim\limits_{i\rightarrow \infty} A^*(\Theta)(x_i)=A^*(\Theta)(x_0)$.
Suppose the contrary. Then  $\forall \eta > 0 $, $||A^*(\Theta)(x_{i})- A^*(\Theta)(x_0)|| >\eta$ holds for infinitely many $i$.
In the rest of the proof, we assume $\eta< \varepsilon/2$.

Let $\zeta=\Loss(\C_{\Theta}(x_0+A^*(\Theta)(x_0)), y_0)-\max_{\alpha\in \B_\epsilon, ||\alpha-A^*(\Theta)(x)||> \eta}\Loss(\C_{\Theta}(x_0+\alpha), y_0)$.
Since $\eta< \varepsilon/2$,
$\{\alpha\in \Be\,:\, ||\alpha-A^*(\Theta)(x)||> \eta\}\ne \emptyset$.
From the uniqueness of $A^*(\Theta)$, we have $\varepsilon>0$.
%
%
%By Lemma \ref{lm-con12}, $\C_{\Theta}$ is Lipschitz for $x$,
%and let $\Lambda$ be the Lipschitz constant.
%
By the convergence of $\{x_i\}_{i=1}^{\infty}$, $\exists N$, such that when $i> N$, $||x_0 - x_{i}|| < \frac{\varepsilon}{3\Psi\Lambda}$. There exists a $k> N$ such that $||A^*(\Theta)(x_{k})- A^*(\Theta)(x_0)|| >\eta$.
Then
    \begin{equation*}
      \renewcommand{\arraystretch}{1.4}
    \begin{array}{ll}
        \Loss(\C_{\Theta}(x_k+A^*(\Theta)(x_0)), y_0) &\ge
        \Loss(\C_{\Theta}(x_0+A^*(\Theta)(x_0)), y_0)-\Psi\Lambda||x_0-x_k||\\
        &\ge \Loss(\C_{\Theta}(x_0+A^*(\Theta)(x_k)), y_0)-\Psi\Lambda||x_0-x_k||+\zeta\\
        &\ge \Loss(\C_{\Theta}(x_k+A^*(\Theta)(x_k)), y_0)-2\Psi\Lambda||x_0-x_k||+\zeta\\
        &> \Loss(\C_{\Theta}(x_k+A^*(\Theta)(x_k)), y_0)+\zeta/3\\
        &> \Loss(\C_{\Theta}(x_k+A^*(\Theta)(x_k)), y_0)
         \end{array}
    \end{equation*}
which contradicts to the definition of $A^*(\Theta)(x_{k})$. Hence $A^*(\Theta)(x)$ is continuous on $x$.
\end{proof}

\begin{lemma}
\label{lm-con23}
Under Assumption ${A}_1$,
%. Assume $\Loss$ is Lipschitz continuous with Lipschitz constant $C$ and
%$\gamma(\Theta)=\{A^*(\Theta)\}$ has a unique element.
%
%for any $\Theta \in \S_c$, $A^*(\Theta)(x)$ is continuous on $\X$,
%
%
$\psi(\Theta)=\varphi_s(\Theta,A^*(\Theta)):\S_c\to\R$ is continuous on $\Theta$.
%for any $A^*(\Theta)\in\gamma_s(\Theta)$.
\end{lemma}
\begin{proof}
We will prove that for any $\zeta>0$, $\exists \delta >0$, such that if
$||\Theta_1-\Theta_2||_2\le \delta$ then $|\varphi_s(\Theta_1, A^*(\Theta_1))-\varphi_s(\Theta_2, A^*(\Theta_2))|\le \zeta.$
Let $\delta= \frac{\zeta}{\Psi\Delta}$.
Then for any $x$,
    \begin{equation*}
    \begin{array}{ll}
     \Loss(\C_{\Theta_1}(x+A^*(\Theta_1)(x)))
    &\le \Loss(\C_{\Theta_2}(x+A^*(\Theta_1)(x)))+\Psi\Delta\delta\\
    &\le
%    \stackrel{\eqref{eq-SEs}}{\le}
  \Loss(\C_{\Theta_2}(x+A^*(\Theta_2)(x)))+\Psi\Delta\delta\\
    & =\Loss(\C_{\Theta_2}(x+A^*(\Theta_2)(x)))+\zeta.
    \end{array}
    \end{equation*}
By exchanging $\Theta_1$ and $\Theta_2$, we have $|\Loss(\C_{\Theta_1}(x+A^*(\Theta_1)(x)))-\Loss(\C_{\Theta_2}(x+A^*(\Theta_2)(x)))|\le \zeta$.
Then $|\varphi_s(\Theta_1, A(\Theta_1))-\varphi_s(\Theta_2, A(\Theta_2))|\le \zeta.$ Thus $\varphi_s(\Theta, A^*(\Theta))$ is continuous on $\Theta$.
\end{proof}

\begin{lemma}
\label{lm-con22}
Under Assumption ${A}_1$,
 $A^*(\Theta):\S_c\to\S_a$ is continuous.
\end{lemma}
\begin{proof}
It suffices to prove that when $\{\Theta_n\}_{n=1}^{\infty}$ converges to $\Theta_0$, $\lim\limits_{n\rightarrow\infty}A^*(\Theta_n)=A^*(\Theta_0)$.
Suppose the contrary. Then there exist $x\in \X$ and $\eta>0$ such that $||A^*(\Theta_n)(x)-A^*(\Theta_0)(x)||> \eta$ holds for infinitely $n$.
We assume  $\eta< \varepsilon/2$.

Let $\zeta=\Loss(\C_{\Theta_0}(x+A^*(\Theta_0)(x)), y)-\max_{\alpha\in \B_\epsilon, ||\alpha-A^*(\Theta_0)(x)||> \eta}\Loss(\C_{\Theta_0}(x+\alpha), y)$.
It is clear that $\zeta >0$.
There exists an $N\in \N_{+}$, such that for any $n>N$,
we have $||\Theta_n - \Theta_0||_2 < \frac{\varepsilon}{2\Psi\Delta}$
and $|\Loss(\C_{\Theta_0}(x+A^*(\Theta_0)(x)), y)-\Loss(\C_{\Theta_n}(x+A^*(\Theta_n)(x)), y)|< \frac{\varepsilon}{2}$ by Lemma \ref{lm-con23}.
%    where $\Delta$ is from Lemma  \ref{lm-con12}.
There exists a $\ j >N$,  $||A^*(\Theta_j)(x)-A^*(\Theta_0)(x)||> \eta$.
Then
    \begin{equation*}
    \begin{array}{ll}
     \Loss(\C_{\Theta_0}(x+A^*(\Theta_0)(x)), y)
    &\ge \Loss(\C_{\Theta_0}(x+A^*(\Theta_j)(x)), y)+\zeta\\
    &\ge \Loss(\C_{\Theta_j}(x+A^*(\Theta_j)(x)), y)+\zeta - \Psi\Delta||\Theta_j-\Theta_0||_2\\
    &>\Loss(\C_{\Theta_j}(x+A^*(\Theta_j)(x)),y)+\frac{\zeta}{2}
    \end{array}
    \end{equation*}
which contradicts to $|\Loss(\C_{\Theta_0}(x+A^*(\Theta_0)(x)), y)-\Loss(\C_{\Theta_j}(x+A^*(\Theta_j)(x)), y)|< \frac{\varepsilon}{2}$.
Thus for any $x, \eta>0$, there exists an $N$ such that for $n>N$, $||A^*(\Theta_n)(x)-A^*(\Theta_0)(x)||\le\eta$ holds, that is,
$\lim \limits_{n\rightarrow\infty} ||A^*(\Theta_n)-A^*(\Theta_0)||_\infty = 0$, which means $A^*(\Theta)$ is continuous on $\Theta$.
\end{proof}

\begin{prop}
\label{prop-ga11}
Under Assumption ${A}_1$,
%    Assume $\Loss$ is continuous with Lipschitz constant $C$, and $\gamma(\Theta)$ has unique element.
$\Gamma_s$  defined in \eqref{eq-Gamma} is a compact set in $\S_c\times\S_a$.
\end{prop}
\begin{proof}
Given a sequence $\{(\Theta_n, A^*(\Theta_n))\}_{n=1}^\infty$ in $\Gamma_s$, since $\S_a$ is compact, there exists a subsequence $\{\Theta_{i_n}\}_{n=1}^\infty$ converges to $\Theta_0$, that is, $\lim\limits_{n\rightarrow\infty}\Theta_{i_n}=\Theta_0$.
By Lemma \ref{lm-con22}, $A^*(\Theta)$ is continuous on $\Theta$, then $\lim\limits_{n\rightarrow\infty}A^*(\Theta_{i_n})=A^*(\Theta_0)$. Hence $\{(\Theta_{i_n}, A^*(\Theta_{i_n}))\}_{n=1}^\infty$ is subsequence converging to $(\Theta_0, A^*(\Theta_{0}))$.
By Lemma \ref{lm-s13}, $\Gamma_s$ is closed, thus $(\Theta_0, A^*(\Theta_{0}))\in\Gamma_s$
and $\Gamma_s$ is compact.
\end{proof}

\section{A Stackelberg game to achieve  maximal adversarial accuracy}
\label{sec-maxaa}

The {\em adversarial accuracy} of a DNN $\C$ with respect to an attack radius $\varepsilon$ is
\begin{equation}
\label{eq-arm12}
\begin{array}{l}
\AA_{\D}(\C,\varepsilon)=\PR_{(x,y)\sim \D}\,( \forall \overline{x} \in\Bxe\, (\widehat{\C}(\overline{x})=y))
\end{array}
\end{equation}
which is the most widely used robustness measurement for DNNs.
Comparing to the robustness measurement $\AR_{\D}$ in \eqref{eq-arm122},
$\AA_{\D}(\C,\varepsilon)$ does not depends on the loss function.
In this section, we will show that adversarial training with the
Carlini-Wagner loss function will give a DNN with the optimal adversarial accuracy.

We first introduce a new game.
Denote $\G_a$ to be the two person zero-sum minmax Stackelberg game with the Classifier as the leader, the Adversary as the follower, and
\begin{equation}
\label{eq-ofunca}
\varphi_a(\Theta,A)=
\EP_{(x,y)\sim \D}\, \Loss_A(\C_{\Theta}(x+A(x)), y).
\end{equation}
as the payoff function, where the loss function is defined as
\begin{equation}
\label{eq-lossA}
\Loss_A(\C(x), y) =
\begin{cases}
0 & \Loss_\CW(\C(x),y)\ge0\\
-1  & \Loss_\CW(\C(x),y)<0\\
\end{cases}
\end{equation}
and $\Loss_\CW$ is the Carlini-Wagner loss function  defined in \eqref{eq-loss12}.

For game $\G_a$,  $\gamma$ and $\Gamma$ defined in \eqref{eq-gamma} and \eqref{eq-Gamma} are
\begin{equation}
\label{eq-gamma1a}
\begin{array}{ccl}
\gamma_a(\Theta)&=&\{  \argmax_{A\in\S_a} \varphi_a(\Theta, A)\}\hbox{ for } \Theta\in\S_c\\
\Gamma_a &=& \{(\Theta,A)\,:\, \Theta\in\S_c, A\in\gamma_a(\Theta)\}.\\
\end{array}
\end{equation}

%the {\em margin loss} introduced by Carlini-Wagner~\cite{CW1}
%\begin{equation}
%\label{eq-ML}
%
%$\Loss_M(\F_{\Theta},x)=- \min_{l\in[m],l\ne l_x} (\F_{l_x}(x)-\F_{l}(x))=\max_{l\in[m],l\ne l_x}\F_{l}(x) -\F_{l_x}(x)$
%\end{equation}
%

\begin{lemma}
\label{lm-ar12}
Let $A_a \in \gamma_a(\Theta)$.
Then
$\varphi_a(\Theta,A_a) = -\AA_{\D}(C_{\Theta},\varepsilon)$.
\end{lemma}
\begin{proof}
Note that $\Loss_A(\C,x,y)=-1$ if and only if $\widehat{\C}(x) = y$
and $\Loss_A(\C,x,y)=0$ if and only if $\widehat{\C}(x) \ne y$
or there exists a $k\ne y$ such that $\C_k(x)=\C_y(x)$.
From $A_a \in \gamma_a(\Theta)$,
$\Loss_A(\C_{\Theta}(x+A_a(x)), y)=-1$ if and only if
$\C_{\Theta}$ is robust over $\Bxe$,
or equivalently, $\widehat{\C}_{\Theta}(\overline{x}) = y$
for any $\overline{x}\in\Bxe$.
Then
$\varphi_s(\Theta,A_a)=
\EP_{(x,y)\sim \D}\, \Loss_A(\C_{\Theta}(x+A_a(x)), y)
=-\AA_{\D}(C_{\Theta},\varepsilon)$.
\end{proof}

\begin{lemma}
\label{lm-ar11}
$\gamma_a(\Theta)\ne\emptyset$ and $A^*\in\gamma_a(\Theta)$ if and only if
$A^*(x) \in\{\argmax_{\overline{x}\in\B(x,\varepsilon)}\,\Loss_A(\C_{\Theta}(\overline{x}), y) \}$
for all $(x,y)\sim\D$.
\end{lemma}
\begin{proof}
We first show that  $\gamma(\Theta,x)=\{\argmax_{\overline{x}\in\B(x,\varepsilon)}\,\Loss_A(\C_{\Theta}(\overline{x}), y) \}\ne\emptyset$ and lemma follows from this.
Let $x^* \in \{ \argmax_{\overline{x}\in\B(x,\varepsilon)}\,$ $\Loss_\CW(\C_{\Theta}(\overline{x}), y)\}$.
Then $x^*$ exists, since $\Loss_\CW$ is continuous and $\Bxe$ is compact.
If $\Loss_\CW(\C,x^*,y)\ge0$, then $\Loss_A(\C,x^*,y)=0$ and $x^*\in \gamma(\Theta,x)$.
If $\Loss_\CW(\C,x^*,y)<0$, then $\Loss_A(\C,x^*,y)=-1$ for all
$x^*\in\B(x,\varepsilon)$ and $\B(x,\varepsilon)=\gamma(\Theta,x)$.
In either case, $\gamma(\Theta,x)\ne\emptyset$.
\end{proof}

\begin{lemma}
\label{lm-ar14}
Let  $(\Theta_\CW^*, A_\CW^*)$ be a Stackelberg equilibrium of game $\G_s$
when the loss function is $\Loss_\CW$ defined in \eqref{eq-loss12}.
Then $(\Theta_\CW^*, A_\CW^*)$ is a Stackelberg equilibrium of game $\G_a$.
\end{lemma}
\begin{proof}
By Lemma \ref{lm-ar12}, $\gamma_a(\Theta)\ne\emptyset$.
So it suffices to show that
$(\Theta_\CW^*,A_\CW^*)
\in
\argmin_{(\Theta,A(\Theta))\in\Gamma_a}\,$
$\varphi_a(\Theta,A(\Theta))$.
%and
%$A_\CW^*
%\in
%\argmax_{A\in\S_a}\,
%\varphi_a(\Theta_\CW^*,A).$
%
Denote $\gamma_\CW, \Gamma_\CW, \varphi_\CW$ to be
$\gamma_s, \Gamma_s, \varphi_s$, when the loss function $\Loss_CW$ is used.

We first prove
$\gamma_\CW(\Theta)\subset \gamma_a(\Theta)$.
Hence
$\Gamma_\CW\subset \Gamma_a$.
By Lemma \ref{lm-s12},
 $A^*\in\gamma_\CW(\Theta)=\{  \argmax_{A\in\S_a} \varphi_\CW(\Theta, A)\}$ if and only if
$A_\CW^*(x) \in \gamma_\CW(\Theta,x,y)=\{\argmax_{A(x)\in\Be}\,\Loss_\CW(\C_{\Theta}(x+A(x)), y) \}$.
By Lemma \ref{lm-ar11},
 $A^*\in\gamma_a(\Theta)$ if and only if
$A_a^*(x) \in \gamma_a(\Theta,x,y)=\{\argmax_{A(x)\in\Be}\,$ $\Loss_a(\C_{\Theta}(x+A(x)), y) \}$.
%for all $(x,y)\sim\D$.
%
Since $\Loss_\CW(\C_{\Theta}(x+A_1), y) \le \Loss_\CW(\C_{\Theta}(x+A_2), y)$
implies
$\Loss_a(\C_{\Theta}(x+A_1), y) \le \Loss_a(\C_{\Theta}(x+A_2), y)$,
we have
$\gamma_\CW(\Theta,x,y)\subset \gamma_a(\Theta,x,y)$.
Then
 $A^*\in\gamma_\CW(\Theta)$
 implies
  $A^*\in\gamma_a(\Theta)$.

We next prove
\begin{equation}
\label{eq-pr11}
\{\varphi_a(\Theta,A), \forall (\Theta,A)\in\Gamma_a \} = \{\varphi_a(\Theta,A), \forall (\Theta,A)\in\Gamma_\CW\}.
\end{equation}
Since $\Gamma_\CW\subset \Gamma_a$,  it suffices to show
$
\{\varphi_a(\Theta,A), \forall (\Theta,A)\in\Gamma_a \} \subset \{\varphi_a(\Theta,A), \forall (\Theta,A)\in\Gamma_\CW\}.
$
For $(\Theta_a,A_a)\in\Gamma_a$, let $A_\CW(x) \in \argmax_{A\in\Be} \Loss_\CW(\C_{\Theta_a}(x+A),y)$. Then $(\Theta_a,A_\CW)\in\Gamma_\CW$.
We will show that  $\varphi_a(\Theta_a,A_a) = \varphi_a(\Theta_a,A_\CW)$.
By Lemma \ref{lm-ar11},
 $A_a^*\in\gamma_a(\Theta_a)$ if and only if
$$A_a^*(x) \in \gamma_a(\Theta_a,x,y)=\{\argmax_{A(x)\in\Be}\,\Loss_a(\C_{\Theta_a}(x+A(x)), y) \}$$
for all $(x,y)\sim\D$.
If $\Loss_a(\C_{\Theta_a}(x+A_a^*(x)),y)=-1$, then
$\Loss_\CW(\C_{\Theta_a}(x+A),y)<0$ for all $A\in\Be$.
In this case, $\max_{A\in\Be} \Loss_\CW(\C_{\Theta_a}(x+A),y)= \Loss_\CW(\C_{\Theta_a}(x+A_\CW^*(x)),y)<0$
and hence $\Loss_a(\C_{\Theta_a}(x+A_\CW^*(x)),y)=-1$.
If $\Loss_a(\C_{\Theta_a}(x+A_a^*(x)),y)=0$, then $\Loss_\CW(\C_{\Theta_a}(x+A_a^*(x)),y)\ge0$.
In this case, $\max_{A\in\Be} \Loss_\CW(\C_{\Theta_a}(x+A),y)=\Loss_\CW(\C_{\Theta_a}(x+A_\CW^*(x)),y)\ge0$
and hence $\Loss_a(\C_{\Theta_a}(x+A_a^*(x)),y)=\Loss_a(\C_{\Theta_a}(x+A_\CW^*(x)),y)=0$.
Then  we have $\varphi_a(\Theta_a,A_a) = \varphi_a(\Theta_a,A_\CW)$.

By \eqref{eq-pr11},
$(\Theta_\CW^*,A_\CW^*)
\in
\argmin_{(\Theta,A(\Theta))\in\Gamma_\CW}\,
\varphi_a(\Theta,A(\Theta))
=\argmin_{(\Theta,A(\Theta))\in\Gamma_a}\,
\varphi_a(\Theta,A(\Theta)).
$
%
%Also by \eqref{eq-pr11},
%Let $A_\CW^* \in \argmax_{A\in\S_c} \varphi_\CW(\C_{\Theta_\CW^*}(x+A),y)$.
%As shown in Step 3,
%$(\Theta_\CW^*,A_\CW^*)\in\Gamma_a$.
The lemma is proved.
\end{proof}

\begin{theorem}
\label{th-ar12}
Let  $(\Theta_\CW^*, A_\CW^*)$ be a Stackelberg equilibrium of game $\G_s$
when the loss function is $\Loss_\CW$   in \eqref{eq-loss12}.
Then $\C_{\Theta_\CW^*}$ has the largest adversarial accuracy for all DNNs in $\Hyp$ defined in \eqref{eq-HS},
that is $\AA_{\D}(\C_{\Theta_\CW^*},\varepsilon)\ge\AA_{\D}(\C_{\Theta},\varepsilon)$ for any  $\C_{\Theta}\in\Hyp$.
\end{theorem}
\begin{proof}
By Lemma \ref{lm-ar14},  $(\Theta_\CW^*, A_\CW^*)$ be a Stackelberg equilibrium of game $\G_a$.
By Lemma \ref{lm-ar12},
$\AA_{\D}(\C_{\Theta_\CW^*},\varepsilon)
= -\varphi_a(\Theta_\CW^*, \argmax_{A\in\S_a}\varphi_a(\Theta_\CW^*,A))
\ge -\varphi_a(\Theta, \argmax_{A\in\S_a}\varphi_a(\Theta,A))
=\AA_{\D}$ $(\C_{\Theta},\varepsilon).
$
The theorem is proved.
\end{proof}

\begin{remark}
\label{rem-s20}
By Theorems \ref{th-s11} and  \ref{th-ar12}, adversarial training using the loss function
$\Loss_\CW$ gives a DNN which has the largest adversarial accuracy for
all DNNs in the hypothesis space $\H$,
which answers Question {$\mathbf{Q_1}$} positively for the hypothesis space $\H$.
\end{remark}

\section{Trade-off between robustness and accuracy}
\label{sec-trade}
%The property of trade-off result between robustness and accuracy for a given DNN
%has been studied extensively~\cite{trade1,trade2}.
%
In this section,  we give  trade-off results between the robustness and the accuracy
in adversarial deep learning from  game theoretical viewpoint.

\subsection{Improve accuracy under maximal adversarial accuracy}

By Remarks \ref{rem-s10} and \ref{rem-s20}, adversarial training computes the DNNs with the best robustness measurement.
%adversarial accuracy.
%
A nature question is whether we can increase the accuracy of the DNN
and still keep the maximal adversarial accuracy.
That is, consider the bi-level optimization problem.
\begin{equation}
\label{eq-TR00}
\begin{array}{lcl}
\Theta_o^*
&=&
\argmin_{\Theta^*}\,  \varphi_0(\Theta^*)\\
% \sum_{i=1}^N \Loss(\C_{\Theta}(\rho_i( {x}_i^*(\Theta)), y_i).
%
&&\hbox{subject to}\\
&&\Theta_s^*=\argmin_{\Theta\in\S_c}\,\mymax_{A\in\S_a} \varphi_s(\Theta,A)\\
\end{array}
\end{equation}
where $\varphi_0$ and $\varphi_s$ are defined in \eqref{eq-LS0} and \eqref{eq-LS1}, respectively.

From Remark \ref{rem-uni}, if using the loss function $\Loss_\CW$,
$\gamma_s(\Theta)$ contains a unique solution and
$\Theta_s^*$ is unique in the generic case.
%For details about this, pleaser refer to Assumption ${{\mathbf A}_1}$ in Section \ref{sec-gamma}.
In this case, we cannot increase the accuracy of the DNN
when  keeping the maximal  robust measure $\AR_\D$.

A  more interesting case is to consider   game $\G_a$ defined in section \ref{sec-maxaa},
which uses the loss function $\Loss_A$ defined in \eqref{eq-lossA}.

We first introduce an assumption.
%The trained parameters of $\C$ can be considered as random values near a local minimum or a
%saddle point of the loss function.
%
We  train  $\C_\Theta$ with stochastic gradient descent starting from a randomly choosing initial point, and most probably will terminate at a random point in the neighborhood of a minimal point or a saddle point of the loss function.
Therefore, the following assumption is valid for almost all trained DNNs~\cite{yu-l2}.

{\bf Assumption $A_2$}. The  parameters of a trained $\C_\Theta$ are {\em random values}.

We now estimate the possible values of $\Theta_s^*$ in \eqref{eq-TR00}.
Suppose a finite data set $T=\{(x_i,y_i)\}_{i=1}^N$ is chosen iid from the distribution $\D$,
which are used to train the network. Then it can be shown that the game $\G_a$
with payoff function \eqref{eq-ofunca} and trained with $T$
has a Stackelberg equilibrium $(\Theta_a^*,A_a^*)$ (See section \ref{sec-emp} for more details).
With these notations, we have
\begin{prop}
\label{pr-to11}
Under Assumption $A_2$, there exists a $\nu\in\R_+$ such that for all $\Theta_a^\circ\in\R^K$
satisfying $||\Theta_a^\circ- \Theta_a^*||< \nu$,
game $\G_a$   has a Stackelberg equilibrium $(\Theta_a^\circ,A_a^\circ)$.
\end{prop}
\begin{proof}
Denote $\phi(\Theta,x)=\Loss_\CW(\C_{\Theta}(x),y)$ for a fixed $y$.
Let $x^* \in \{ \argmax_{\overline{x}\in\B(x_i,\varepsilon)}\,\phi(\Theta_a^*,\overline{x})$.
If $\phi(\Theta_a^*,x^*)<0$, then
$\phi(\Theta_a^*,\overline{x})<0$ for all $\overline{x}\in\B(x_i,\varepsilon)$.
Since $\B(x_i,\varepsilon)$ is compact and $\phi(\Theta,x)$ is continuous,
there exists a $\nu_i\in\R_+$
such that
$\phi(\Theta_a^*+\Delta,\overline{x})<0$ for all $\overline{x}\in\B(x_i,\varepsilon)$
and all $\Delta\in\R^K$ satisfying $||\Delta||< \nu_i$.
Without loss of generality, we can assume $\Theta_a^*+\Delta\in\S_c$.
It is easy to construct the best response of the Adversary in this case for $\Theta_a^\circ=\Theta_a^*+\Delta$: $A_a^\circ(x_i)$ can be any point in $\B(x_i,\varepsilon)$.
If $\phi(\Theta_a^*,x^*)>0$, then $S_i(\Theta_a^*) = \{x\in \B(x_i,\varepsilon)\,:\, \Loss_\CW(\C_{\Theta_a^*}(x),y)\le0\}$
is a compact set of dimension $m$, since $\Loss_\CW(\C_{\Theta_a^*}(x),y)$ is piecewise linear in $x$.
If $\nu_i$ is small enough, then $S_i(\Theta_a^*+\Delta)$
is also a compact set of dimension $m$
for all $\Delta\in\R^K$ and $||\Delta||< \nu_i$.
In this case, $A_a^\circ(x_i)$ can be any point in $\S_i(\Theta_a^*+\Delta)$.

By Assumption $A_2$,the trained parameters of $\C$ are random values.
$\phi(\Theta_a^*,x^*)=\Loss_\CW(\C_{\Theta_a^*}(x^*),$ $y)=0$ implies that $\C_{\Theta_a^*,i}(x^*)=\C_{\Theta_a^*,j}(x^*)$ for
$i\ne j$, which gives an algebraic relation among the parameters of $\C_{\Theta}$.
This imposes an extra algebraic relation among the random parameters and thus will not happen
under  Assumption $A_2$.
So we have $\phi(\Theta,x^*)\ne0$ under  Assumption $A_2$.

Let $\nu= \min_{i=1}^N \nu_i>0$. Then for   $||\Theta_a^\circ- \Theta_a^*||< \nu$,
there exists an $A_a^\circ\in\S_c$ such that
$\varphi_a(\Theta_a^\circ,A_a^\circ)=\varphi_a(\Theta_a^*,A_a^*)$,
where $\varphi_a$ is defined in \eqref{eq-ofunca}.
Since $(\Theta_a^*,A_a^*)$ is a Stackelberg equilibrium for game $\G_a$,
so is $(\Theta_a^\circ,A_a^\circ)$. The proposition is proved.
\end{proof}

By Proposition \ref{pr-to11}, $\Theta_s^*$  in \eqref{eq-TR00}
takes values in a $K$-dimensional set.
As a consequence, there exist rooms for increase the accuracy
under the maximal adversarial accuracy.
%We use the following experiments to show this.

\begin{example}
\label{ex-10}
We use numerical experiments to show that it is possible to further increase the accuracy under the maximal adversarial accuracy.
Two small CNNs with respectively 3 and 4 hidden layers are used,
which have structures
$(8*3*3),(16*3*3),(32*3*3)$ and
$(32*3*3),(64*3*3),(128*3*3),(128*3*3)$, respectively.
We use loss function $\Loss_\CW$ to achieve maximal adversarial accuracy
and the results are given in the columns 1-0 and 2-0 in Table \ref{tab-10}.
We then retrain the CNNs using the normal loss function in \eqref{eq-LS0} to
increase the accuracy.
In order to keep the maximal adversarial accuracy fixed, the change
of the parameters are limited to $i\%$ for $i=1,2,3$ and the results are
given in columns 1-i and 2-i, respectively.
We can see that
the adversarial accuracies are barely changed
(up to $0.06\%$ and $0.02\%$  for networks 1 and 2),
but the accuracies are increased evidently
(up to $1.11\%$ and $2.252\%$ for networks 1 and 2).
\begin{table}[H]
\caption{Increase the accuracy (AC) under the condition of maximal adversarial accuracy (AA)
for CIFRA-10. The attack radius is $8/255$ and 50000 samples are used.
%We give the numbers satisfying the condition among the 50000 samples.
%In the row starting with AC, we give the number  of samples among 50000 samples,
%for which the DNN gives the correct label.
%In the row starting with AA, we give the number  of samples among 50000 samples,
%for which the DNN is robust.
}
\label{tab-10}
\centering
\vskip2pt
%\scriptsize
\setlength{\tabcolsep}{5.1pt}
\begin{tabular}{lccccccccc}
  \hline
  \multirow{2}{*}{}
   & \multicolumn{4}{c}{Network 1}
 &  & \multicolumn{4}{c}{Network 2} \\
  \cline{2-5} \cline{7-10}
        & 1-0& 1-1 & 1-2& 1-3    &
        & 2-0& 2-1 & 2-2& 2-3  \\
% \hline
AC (\%) &
45.718 & 46.762 & 46.814 & 46.828 &
&72.156 & 75.284 & 75.344 & 75.408  \\
%AC &
%22859 & 23381 & 23407 & 23414 &
%&36078 & 37642 & 37672 & 37704  \\
%
AA (\%) &
 29.018 & 28.996&  28.98 & 28.958 &
&40.08 & 40.076 & 40.036 & 40.06  \\
%
%AA &
%14509 & 14498&  14490 & 14479 &
%&20040 & 20038 & 20018 & 20030  \\
%
  \hline
\end{tabular}
\end{table}
\end{example}

\subsection{An effective trade-off method}

The bi-level optimization problem \eqref{eq-TR00} is in general difficult to solve,
especially when keeping the maximal adversarial accuracy as mentioned in the proof
of Proposition \ref{pr-to11}.
A natural way to  train a robust and more accurate DNN is to do
adversarial training with the following objective function
\begin{equation}
\label{eq-AT20}
\begin{array}{lcl}
\varphi_t({\Theta},A)
&=& \varphi_s({\Theta},A)+ \lambda\varphi_0({\Theta})\\
\end{array}
\end{equation}
where $\lambda>0$ is a small hyperparameter,
$\varphi_0$ and $\varphi_s$ are defined in \eqref{eq-LS0} and \eqref{eq-LS1}, respectively.
Problem \eqref{eq-AT20} is also often used as an approximate way to solve
\eqref{eq-TR00}.
We will prove a trade-off result in this setting.

Similar to Theorem \ref{th-s11}, adversarial training with loss function
\eqref{eq-AT20} can be considered as a Stackelberg game $\G_t$ with $\varphi_t$ as the payoff function.
Then we have the following trade-off result.

\begin{prop}
\label{prop-s1}
Let $(\Theta_s^*,A_s^*)$ and $(\Theta_t^*,A_t^*)$ be the Stackelberg equilibria
of the zero-sum sequential games with $\varphi_s$ and $\varphi_t$ as the payoff functions, respectively.
Then
$$
\AA_{\D}(\C_{\Theta_s^*}, \varepsilon) \ge \AA_{\D}(\C_{\Theta_t^*}, \varepsilon),
\varphi_s(\Theta_s^*,A_s^*) \le \varphi_s(\Theta_t^*,A_t^*)
\hbox{ and }
\varphi_0(\Theta_s^*) \ge \varphi_0(\Theta_t^*)
$$
that is, the network $\C_{\Theta_s^*}$  is more robust but less accurate
than $\C_{\Theta_t^*}$ measured by $\varphi_0$.
\end{prop}
\begin{proof}
$\AA_{\D}(\C_{\Theta_s^*}, \varepsilon) \ge \AA_{\D}(\C_{\Theta_t^*}, \varepsilon)$
is a consequence of Theorem \ref{th-ar12}.
Since $(\Theta_t^*,A_t^*)$ is a Stackelberg equilibrium of game $\G_t$, we have
\begin{eqnarray}
\Theta_t^* &\in&
%\mymin_{\Theta_c\in\S_c} \mymax_{A\in \S_a}\, \varphi_2(\Theta,A)\\
\argmin_{\Theta\in\S_c}\, \varphi_t(\Theta, \argmax_{A\in S_a}\, \varphi_t(\Theta,A))\label{eq-pr-32}\\
A_t^* &\in& \argmax_{A\in \S_a}\, \varphi_t(\Theta_t^*,A)\nonumber\\
     &=& \argmax_{A\in \S_a}\, (\varphi_s(\Theta_t^*,A) + \lambda \varphi_0(\Theta_t^*))\label{eq-pr-33}\\
     &=& \argmax_{A\in \S_a}\, \varphi_s(\Theta_t^*,A)\nonumber
%\end{array}
\end{eqnarray}
where the last equality is due to the fact that
$\varphi_0(\Theta_t^*)$ is free of $A$.
Then, from \eqref{eq-pr-31},
\begin{eqnarray}
\label{eq-pr-30}
\varphi_s(\Theta_s^*,A_s^*)
&=& \varphi_s(\Theta_s^*, \argmax_{A\in S_a}\, \varphi_s(\Theta_s^*,A))\nonumber\\
&\le& \varphi_s(\Theta_t^*, \argmax_{A\in S_a}\, \varphi_s(\Theta_t^*,A))\\
&\le& \mymax_{A\in\S_a}\varphi_s(\Theta_t^*,A)= \varphi_s(\Theta_t^*,A_t^*).\nonumber
\end{eqnarray}
The last equality comes from \eqref{eq-pr-33}.
From \eqref{eq-pr-32},
\begin{eqnarray}
\label{eq-pr-35}
\varphi_t(\Theta_t^*,A_t^*)
\le \varphi_t(\Theta_s^*, \argmax_{A\in S_a}\, \varphi_t(\Theta_s^*,A))
\le \mymax_{A\in\S_a}\varphi_t(\Theta_s^*,A)
= \varphi_t(\Theta_s^*,A_s^*).
\end{eqnarray}
Adding  inequalities \eqref{eq-pr-30} and \eqref{eq-pr-35}, we obtain
$\varphi_0(\Theta_s^*) \ge \varphi_0(\Theta_t^*)$. The proposition is proved.
\end{proof}
%
%We thus proved a trade-off theorem, meaning that in order to increase the
%accuracy, the robustness measured by the loss function $\varphi_t$ must
%be decreased.
%
%As consequence of Proposition \ref{prop-s1} and Theorem \ref{th-ar12}, we have
%\begin{remark}
%Note that proposition \ref{prop-s1} implies that it is impossible
%to increase the accua
%
%doe not contradicts to proposition \ref{pr-to11},
%because
%
%By Theorem \ref{th-ar12},  when $\Loss_\CW$ is used as the loss function,
%the inequality $\varphi_s(\Theta_s^*,A_s^*) \le \varphi_s(\Theta_t^*,A_t^*)$ in the above proposition
%can be replace by
%%
%$\AA_{\D}(\C_{\Theta_s^*}, \varepsilon) \ge \AA_{\D}(\C_{\Theta_t^*}, \varepsilon)$,
%%\hbox{ and }\AC(\C_{\Theta_s^*}) \le \AC(\C_{\Theta_t^*})$$
%that is, the network $\C_{\Theta_s^*}$ has larger adversarial accuracy than $\C_{\Theta_t^*}$.
%\end{remark}
%\end{cor}
%\begin{proof}
%From \eqref{eq-pr-33}, $A_t^* \in \gamma_a(\Theta_t^*)$.
%By Lemma \ref{lm-ar12}, $\AA_{\D}(\C_{\Theta_t^*}, \varepsilon)=-\varphi_s(\Theta_t^*,A_t^*)$.
%By Theorem \ref{th-ar12}, $\AA_{\D}(\C_{\Theta_s^*}, \varepsilon)=-\varphi_s(\Theta_s^*,A_s^*)$.
%Then the corollary comes from Proposition \ref{prop-s1}.
%\end{proof}

Note that this trade-off result is quite different from the trade-off theorem in \cite{trade1} in that, our result is  for any data set, while
the result in \cite{trade1} is for a specifically designed data set.

\section{Comparing three types of games for adversarial deep learning}
\label{sec-emp}
In this section, we compare three types of games for adversarial deep learning
when the data  $T=\{(x_i,y_i)\}_{i=1}^N\subset\I^n\times\Y$ are a finite number of samples chosen iid from the distribution $\D$.

In this case, the strategy space for the Classifier is still
$\S_c$ in \eqref{eq-Sc}.
The strategy space for the Adversary becomes much simpler:
\begin{equation}
\label{eq-SaE}
\S_a = \myprod_{i=1}^N \{(\overline{x}_i,y_i)\,:\, ||\overline{x}_i-x_i||\le\varepsilon \}\subset (\I_{\varepsilon}^{n}\times\Y)^N
\end{equation}
where $\I_{\varepsilon}=[-\varepsilon,1+\varepsilon]$.
%For simplification of presentations, we still assume that $\overline{x}_i\in \I^n$,
%since we cam assume that $x_i\in[\varepsilon,1-\varepsilon]^n$ when $\varepsilon\ll1$.
%
%
For $\Theta\in\S_c$ and  $A=((\overline{x}_i,y_i))_{i=1}^N\in\S_a$,
the {\em empirical adversarial loss} is
\begin{equation}
\label{eq-ofunc1E}
\varphi_T(\Theta,A)=
\frac1N \mysum_{i=1}^N\Loss(\C_{\Theta}(\overline{x}_i), y_i).
\end{equation}
%which is a continuous function in $\Theta$ and $\overline{x}_i$.
%
We consider three games.

{\bf The adversarial training game $\G_1$},
which is the zero-sum minmax sequential game  with the Classifier as the leader,
the Adversary as the follower, and $\varphi_T(\Theta,A)$ as the payoff function,
that is, to solve the following minmax problem
\begin{equation}
\label{eq-AT10E}
\Theta_1^* =
\argmin_{\Theta\in\S_c} \mymax_{A \in \S_a}
\, \varphi_T(\Theta,A)
%\frac1N \sum_{i=1}^N\Loss(\C_{\Theta}(\overline{x}_i), y_i)
\end{equation}
which is clearly equivalent to the adversarial training.
%where the inner maximization loop is solved with the gradient based PGD method.
%The robustness measure in \eqref{eq-arm12} is difficult to compute precisely,
%and we may use the following approximation to $\AR(\C,\varepsilon)$ on a test set
%$T = \{(x_i,y_i)\}_{i=1}^N$.
By Theorem \ref{th-SE}, game $\G_1$ has a Stackelberg equilibrium $(\Theta_1^*,A_1^*)$,
since $\S_c$ and $\S_a$ are compact and $\varphi_T(\Theta,A)$ is continuous.
Similar to section \ref{sec-maxaa}, it can be shown that this game gives
a DNN with the largest adversarial accuracy for the data set $T$,
when the loss function is $\Loss_\CW$.

{\bf The universal adversary game $\G_2$},
which is the zero-sum maxmin sequential game  with the Adversary  as the leader
and the Classifier as the follower, that is, to solve the following maxmin problem
\begin{equation}
\label{eq-AT10U}
\A_2^* =
\argmax_{A \in \S_a} \mymin_{\Theta\in\S_c}
\, \varphi_T(\Theta,A)
%\frac1N \sum_{i=1}^N\Loss(\C_{\Theta}(\overline{x}_i), y_i)
\end{equation}
By Theorem \ref{th-SE}, game $\G_2$ has a Stackelberg equilibrium $(\Theta_2^*,A_2^*)$.
The solution $(\Theta_2^*,A_2^*)$ of this game is to compute
the optimal {\em universal adversarial attack} for the given hypothesis space $\Hyp$ in \eqref{eq-HS},
that is, $A_2^*(x)$ is the best adversary for any $(x,y)\sim\D$
and for all DNNs in $\Hyp$.
It is clear that $\A_2^*$ is the optimal attack to the so-called
{\em nobox model} proposed in \cite{game-aeg1},
that is, nobox model has an optimal solution for DNNs with a given structure.
This gives a positive answer to question $\mathbf{Q_2}$ for the hypothesis space $\Hyp$
in \eqref{eq-HS}.

{\bf The simultaneous adversary game $\G_3$}.
We can also  formulate the adversarial deep learning as a simultaneous
game $\G_3$. In this game, the two players and their strategy spaces are the same
as that of game $\G_1$.
The difference is the way to play the game.
In game $\G_3$, the Classifier picks its action
 without knowing the action of the Adversary,
and the Adversary chooses the attacking adversarial samples
without knowing the action of the Classifier.
But, both players know the payoff function.
A point $(\Theta_3^*,A_3^*)\in\S_c\times\S_a$ is called
a  {\em pure strategy Nash equilibrium} of game $\G_3$ if
\begin{equation}
\label{eq-ne10}
\begin{array}{lcl}
\Theta_3^* &=& \argmin_{\Theta\in\S_c} \varphi_T(\Theta,A_3^*) \hbox{ and }
A_3^* =  \argmax_{A\in\S_a} \varphi_T(\Theta_3^*, A).\\
\end{array}
\end{equation}

In general, pure strategy Nash equilibria do not necessarily exist,
and mixed strategy Nash equilibria are usually considered.
{\em Mixed strategies} for the Classifier and the Adversary are two probability distributions
$$\widetilde{\Theta}:\S_c\rightarrow \I \hbox{ and }\widetilde{A}:\S_a\rightarrow \I$$
for $\Theta$ and $A$, respectively.
For a mixed strategy $(\widetilde{\Theta},\widetilde{A})$,
the payoff function is
\begin{equation}
\label{eq-mpof10}
\varphi_T(\widetilde{\Theta},\widetilde{A}) =
\EP_{\Theta\sim\widetilde{\Theta}}\,\EP_{A\sim\widetilde{A}}\,\varphi_T(\Theta, A).
\end{equation}
Denote $\widetilde{\S}_c$ and $\widetilde{\S}_a$ to be the sets of
the mixed strategies for the Classifier and the Adversary, respectively.
Then $(\widetilde{\Theta}_3^*,\widetilde{A}_3^*)\in \widetilde{\S}_c\times \widetilde{\S}_a$
is called {\em a mixed strategy Nash equilibrium} of game $\G_3$ if
\begin{equation}
\label{eq-mne10}
\begin{array}{lcl}
\widetilde{\Theta}_3^* &=& \argmin_{\widetilde{\Theta}\in\widetilde{\S}_c} \, \varphi_T(\widetilde{\Theta},\widetilde{A}_3^*) \hbox{ and }
\widetilde{A}_3^* =
\argmax_{\widetilde{A}\in\widetilde{\S}_a}\,
 \varphi_T(\widetilde{\Theta}_3^*, \widetilde{A}).\\
\end{array}
\end{equation}

Since the strategy spaces of the two players are compact
and the objective function is continuous,
by Glicksberg's theorem~\cite{Glicksberg},
game $\G_3$ has a mixed strategy Nash equilibrium
$(\widetilde{\Theta}_3^*, \widetilde{A}_3^*)$,
and the minmax theorem holds for this equilibrium.

\begin{remark}
By Proposition \ref{prop-ga11},
we can show that, under Assumption $A_1$, game $G_3$ has a mixed strategy when
the data set satisfies a general distribution $\D$.
\end{remark}

\begin{prop}
\label{pr-com12}
Let $(\Theta_i^*,A_i^*)$ be Nash equilibria of games $\G_i$
for $i=1,2,3$, respectively (mixed strategy for $\G_3$).
Then
$$\varphi_T(\Theta_1^*, A_1^*)\ge \varphi_T(\Theta_3^*, A_3^*)\ge\varphi_T(\Theta_2^*, A_2^*).$$
\end{prop}
\begin{proof}
The mixed strategy $(\Theta_3^*,A_3^*)$ can be written as two distributions $\Delta_c:\S_c\rightarrow \I$ and  $\Delta_a:\S_a\rightarrow \I$, respectively.
To prove the first inequality, we have
\begin{eqnarray*}
\label{eq-pr-36}
\varphi_T(\Theta_1^*,A_1^*)
= \EP_{A\sim \Delta_a}\, \varphi_T(\Theta_1^*,A_1^*)
\stackrel{\eqref{eq-pr-31}}{\ge} \EP_{A\sim \Delta_a}\, \varphi_T(\Theta_1^*,A)
=  \varphi_T(\Theta_1^*,A_3^*)
\stackrel{\eqref{eq-mne10}}{\ge}  \varphi_T(\Theta_3^*,A_3^*).
\end{eqnarray*}
For the second inequality, we have
\begin{eqnarray*}
\label{eq-pr-34}
\varphi_T(\Theta_2^*,A_2^*)
=   \EP_{\Theta\sim \Delta_c}\, \varphi_T(\Theta_2^*,A_2^*)
\stackrel{\eqref{eq-AT10U}}{\le} \EP_{\Theta\sim \Delta_c}\,  \varphi_T(\Theta,A_2^*)
=  \varphi_T(\Theta_3^*,A_2^*)
\stackrel{\eqref{eq-mne10}}{\le}  \varphi_T(\Theta_3^*,A_3^*).
\end{eqnarray*}
%Step $S_1$ is due to the fact that $(\Theta_2^*,A_2^*)$ is a Nash equilibrium
The proposition is proved.
\end{proof}

The following example shows that the inequalities in
Proposition \ref{pr-com12} could be strict.
\begin{example}
Consider a two-player zero-sum minmax game with
 payoff matrix
\begin{equation*}
\left(
\begin{array}{cc}
0 & -a \\
-1 & 0\\
\end{array}
\right)
\end{equation*}
where $0<a<1$.
The strategy space for player one  is the rows and its goal is minimize the payoff.
Then, the Stackelberg game with player one as the leader is to solve
the $\minmax$ problem and a Stackelberg equilibrium is $($Row 1, Column 1$)$
with payoff $0$.
The Stackelberg game with player two (column) as the leader is to solve
the $\maxmin$ problem and a Stackelberg equilibrium is $($Row 1, Column 2$)$
with payoff $-a$.
By the well known minmax theorem, the corresponding simultaneous game has no
Nash equilibrium since  $\minmax\ne \maxmin$,
and a mixed strategy Nash equilibrium exists:
the first player plays $(\frac{1}{1+a},\frac{a}{1+a})$
and the second player plays $(\frac{a}{1+a},\frac{1}{1+a})$ with payoff  $- \frac{a}{1+a}$.
We summarize the above discussion as follows:
\begin{equation*}
\begin{array}{ll}
\minmax & \hbox{\rm payoff} =  0\\
\maxmin & \hbox{\rm payoff} =  -a\\
\hbox{\rm Mixed strategy} & \hbox{\rm payoff} = - \frac{a}{1+a}\in(-a,0).
\end{array}
\end{equation*}
%For this example, the two inequalities are valid.
\end{example}

\section{Conclusion}
\label{sec-conc}
In this paper, we give a game theoretical analysis for adversarial deep learning from
a more practical viewpoint.
In previous work, the adversarial deep learning was formulated
as a simultaneous game. In order for the Nash equilibrium to exist,
the strategy spaces for the Classifier and the Adversary
are assumed to be certain convex probability distributions, which are not used in real applications.
In this paper,  the adversarial deep learning is formulated
as a sequential game with the Classifier as the leader and
the Adversary as the follower.
In this case, we show that the game has Stackelberg equilibria
when the strategy space for the  classifier is
DNNs with given width and depth, just like people do in practice.

We prove that Stackelberg equilibria for such a sequential game
is the same as the DNNs obtained with adversarial training.
Furthermore, if the  margin loss introduced by Carlini-Wagner
is used as the payoff function, the   equilibrium DNN
has the largest adversarial accuracy and is thus the provable optimal defence.
Based on this  approach, we also give theoretical analysis for
other important issues such as the tradeoff between robustness
and the accuracy, and the generation of optimal universal adversaries.

For future research, it is desirable to develop
practical methods to use mixed strategy in deep learning,
since it is proved that such strategy has more power
than pure strategy when the depth and width of the DNNs are fixed.
It is also interesting to analysis the properties
of the Nash equilibria for adversarial deep learning,
such as whether the equilibria are regular or essential~\cite{Damme-game,Wu-game}?
Finally, we can use game theory to analyze  other adversarial problems
in deep learning.

\end{document}